\documentclass{article}

\usepackage{times}
\usepackage{graphicx} % more modern
\usepackage{subfigure}

\usepackage{natbib}

\usepackage{algorithm}
\usepackage{algorithmic}
\usepackage[algo2e,linesnumbered,lined,boxed,vlined]{algorithm2e}

\usepackage{amsfonts}
\usepackage{amssymb}

\usepackage{hyperref}

\usepackage[accepted]{icml2016}

\usepackage{comment}
\usepackage{amsthm}
\usepackage{mathtools}
\usepackage{tabularx}
\usepackage{multirow}

\newtheorem{lemma}{Lemma}
\newtheorem{proposition}{Proposition}
\newtheorem{corollary}{Corollary}

\def\b{\ensuremath\boldsymbol}

\usepackage{lastpage}
\usepackage{fancyhdr}
\usepackage{url}
\icmltitlerunning{Laplacian-Based Dimensionality Reduction: Tutorial and Survey}

\begin{document}

\AddToShipoutPictureBG*{%
  \AtPageUpperLeft{%
    \setlength\unitlength{1in}%
    \hspace*{\dimexpr0.5\paperwidth\relax}
    \makebox(0,-0.75)[c]{\normalsize {\color{black} To appear as a part of an upcoming textbook on dimensionality reduction and manifold learning.}}
    }}

\twocolumn[
\icmltitle{Laplacian-Based Dimensionality Reduction Including Spectral Clustering, Laplacian Eigenmap, Locality Preserving Projection, Graph Embedding, and Diffusion Map: Tutorial and Survey}

\icmlauthor{Benyamin Ghojogh}{bghojogh@uwaterloo.ca}
\icmladdress{Department of Electrical and Computer Engineering, 
\\Machine Learning Laboratory, University of Waterloo, Waterloo, ON, Canada}
\icmlauthor{Ali Ghodsi}{ali.ghodsi@uwaterloo.ca}
\icmladdress{Department of Statistics and Actuarial Science \& David R. Cheriton School of Computer Science, 
\\Data Analytics Laboratory, University of Waterloo, Waterloo, ON, Canada}
\icmlauthor{Fakhri Karray}{karray@uwaterloo.ca}
\icmladdress{Department of Electrical and Computer Engineering, 
\\Centre for Pattern Analysis and Machine Intelligence, University of Waterloo, Waterloo, ON, Canada}
\icmlauthor{Mark Crowley}{mcrowley@uwaterloo.ca}
\icmladdress{Department of Electrical and Computer Engineering, 
\\Machine Learning Laboratory, University of Waterloo, Waterloo, ON, Canada}

\icmlkeywords{Tutorial}

\vskip 0.3in
]

\begin{abstract}
This is a tutorial and survey paper for nonlinear dimensionality and feature extraction methods which are based on the Laplacian of graph of data. We first introduce adjacency matrix, definition of Laplacian matrix, and the interpretation of Laplacian. Then, we cover the cuts of graph and spectral clustering which applies clustering in a subspace of data. Different optimization variants of Laplacian eigenmap and its out-of-sample extension are explained. Thereafter, we introduce the locality preserving projection and its kernel variant as linear special cases of Laplacian eigenmap. Versions of graph embedding are then explained which are generalized versions of Laplacian eigenmap and locality preserving projection. Finally, diffusion map is introduced which is a method based on Laplacian of data and random walks on the data graph. 
\end{abstract}

\section{Introduction}

Nonlinear dimensionality reduction methods find the manifold of data in a nonlinear way for the sake of feature extraction or data visualization \cite{ghojogh2019feature}. 
Dimensionality reduction methods belong to three broad categories, which are spectral, probabilistic, and neural network-based methods \cite{ghojogh2021data}. 
The spectral methods deal with the graph and geometry of data and usually reduce to an eigenvalue or generalized eigenvalue problem \cite{ghojogh2019eigenvalue}. 
A sub-family of the spectral dimensionality reduction methods is based on the Laplacian of graph of data \cite{merris1994laplacian}. The importance of Laplacian was probably noticed more when researchers gradually proposed spectral clustering methods which made use of Laplacian \cite{weiss1999segmentation,ng2001spectral}. Spectral clustering transforms data onto a low-dimensional subspace to hopefully separate the clusters better. Then, in that discriminating subspace, a simple clustering algorithm is applied. Researchers adopted the idea of spectral clustering to use it for the goal of dimensionality reduction. The result was the Laplacian eigenmap \cite{belkin2001laplacian,belkin2003laplacian}. 
In the meantime, out-of-sample extension techniques were also proposed for spectral clustering and Laplacian eigenmap \cite{bengio2003out,gisbrecht2012out,bunte2012general,gisbrecht2015parametric}.
Later on, people decided to linearize the nonlinear Laplacian eigenmap using linear projection. Hence, Locality Preserving Projection (LPP) was proposed \cite{he2004locality}. A kernel version of LPP was also proposed \cite{cheng2005supervised,li2008kernel} to apply LPP in the feature space rather than input space. 
Afterwards, researchers found out that many spectral methods belong to a family of methods so they tried to propose a generalized spectral method using the Laplacian matrix. Therefore, graph embedding, with its various variants, was proposed \cite{yan2005graph,yan2006graph}. 
It was around that time that another method, which made use of Laplacian and kernel of data, was proposed which revealed the underlying nonlinear manifold of data in a diffusion process. This method is known as the diffusion map \cite{Lafon2004diffusion,coifman2005geometric,coifman2006diffusion}. 
In this paper, we explain the theory and intuitions of these Laplacian-based methods in detail. 
Note that although Locally Linear Embedding (LLE) can also be considered as a Laplacian-based method, we have provided a separate tutorial for it and its variants \cite{ghojogh2020locally}.
Note that there exist some surveys, such as \cite{wiskott2019laplacian}, on the Laplacian-based methods.

The remainder of this paper is as follows. We introduce the adjacency matrix, Laplacian matrix, and their interpretations in Section \ref{section_interpretation_Laplacian}. Sections \ref{section_spectral_clustering} introduces the cuts and clustering for two and multiple number of classes. Laplacian eigenmap and its out-of-sample extension are explained in Section \ref{section_Laplacian_eigenmap}. We introduce LPP and kernel LPP in Section \ref{section_Locality_Preserving_Projection}. Direct, linearized, and kernelized graph embedding are also detailed in Section \ref{section_graph_embedding}. We introduce the diffusion map in Section \ref{section_diffusion_map}. Finally, Section \ref{section_conclusion} concludes the paper. 

\section*{Required Background for the Reader}

This paper assumes that the reader has general knowledge of calculus, linear algebra, and basics of optimization. 

\section{Laplacian Matrix and its Interpretation}\label{section_interpretation_Laplacian}

Most of the methods in this tutorial use the Laplacian matrix of data graph. In this section, we introduce the Laplacian matrix and its interpretation. 

\subsection{Adjacency Matrix}

Assume we have a dataset $\{\b{x}_i\}_{i=1}^n$.
In spectral clustering and related algorithms, we first construct a nearest neighbors graph. There are two versions of constructing the nearest neighbors graph \cite{belkin2001laplacian}:
\begin{enumerate}
\item graph using $\epsilon$-neighborhoods: nodes $i$ and $j$ are connected if $\|\b{x}_i - \b{x}_j\|_2^2 \leq \epsilon$ 
\item $k$-Nearest Neighbors ($k$NN) graph: node $j$ is connected to node $i$ are connected if $\b{x}_j$ is among the $k$ nearest neighbors of $\b{x}_i$
\end{enumerate}
Using the nearest neighbors graph, we construct an adjacency matrix, also called weight matrix or affinity matrix, which is a weighted version of nearest neighbors graph. The weights of adjacency matrix represent the similarity of nodes. The more similar two nodes are, the larger their weight is. Let $\b{W} \in \mathbb{R}^{n \times n}$ denote the adjacency or weight matrix whose $(i,j)$-th element is \cite{ng2001spectral}:
\begin{align}\label{equation_adjacency_matrix}
\b{W}(i,j) := 
\left\{
    \begin{array}{ll}
        w_{ij} & \mbox{if } \b{x}_j \in \text{nearest neighbors of } \b{x}_i, \\
        w_{ij} = 0 & \mbox{if } i = j, \\
        w_{ij} = 0 & \mbox{if } \b{x}_j \not\in \text{nearest neighbors of } \b{x}_i.
    \end{array}
\right.
\end{align}
Note that one can choose a fully connected nearest neighbors graph where all points are among the neighbors of each other. In that case, the third above condition can be omitted \cite{ng2001spectral}.

An example of the weight function for the edge between $\b{x}_i$ and $\b{x}_j$ is the Radial Basis Function (RBF) or heat kernel \cite{belkin2001laplacian}:
\begin{align}\label{equation_RBF_kernel}
w_{ij} = \exp\Big(\!\!-\frac{\|\b{x}_i - \b{x}_j\|_2^2}{2 \sigma^2}\Big),
\end{align}
whose parameter $\sigma^2$ can be set to $0.5$ or one for simplicity. This weight function gets larger, or closer to one, when $\b{x}_i$ and $\b{x}_j$ are closer to each other. 

Another approach for choosing $w_{ij}$ is the simple-minded approach \cite{belkin2001laplacian}:
\begin{align}\label{equation_simple_minded}
w_{ij} := 
\left\{
    \begin{array}{ll}
        1 & \mbox{if } i \text{ and } j \text{ are connected} \\
        & \quad \quad \text{in the nearest neighbors graph}, \\
        0 & \mbox{otherwise. } 
    \end{array}
\right.
\end{align}

\subsection{Laplacian Matrix}\label{section_Laplacian_matrix_definition}

We define the Laplacian matrix of the adjacency matrix, denoted by $\b{L} \in \mathbb{R}^{n \times n}$, as \cite{merris1994laplacian}:
\begin{align}\label{equation_Laplacian_matrix}
\b{L} := \b{D} - \b{W},
\end{align}
where $\mathbb{R}^{n \times n} \ni \b{D} = \text{diag}([d_1, \dots, d_n]^\top)$ is the degree matrix with diagonal elements as:
\begin{align}\label{equation_degree_matrix}
d_i = \sum_{j=1}^n w_{ij},
\end{align}
which is the summation of its corresponding row. 
Therefore, the row summation of Laplacian matrix is zero:
\begin{align}\label{equation_Laplacian_row_sum}
\sum_{j=1}^n \b{L}_{ij} = \b{0}.
\end{align}

It is noteworthy that there exist some other variants of Laplacian matrix such as \cite{weiss1999segmentation,ng2001spectral}:
\begin{align}\label{equation_Laplacian_matrix_2}
\b{L} := \b{D}^{-\alpha} \b{W} \b{D}^{-\alpha},
\end{align}
where $\alpha \geq 0$ is a parameter. A common example is $\alpha=0.5$. 
Some papers have used \cite{weiss1999segmentation,ng2001spectral}:
\begin{align}\label{equation_Laplacian_matrix_D_W_D}
\b{L} := \b{D}^{-1/2} \b{W} \b{D}^{-1/2},
\end{align}
rather than Eq. (\ref{equation_Laplacian_matrix}) for the definition of Laplacian.  
According to \cite{weiss1999segmentation}, Eq. (\ref{equation_Laplacian_matrix_D_W_D}) leads to having the $(i,j)$-th element of Laplacian as:
\begin{align}\label{equation_Laplacian_matrix_W_over_D_D}
\b{L}(i,j) := \frac{\b{W}(i,j)}{\sqrt{\b{D}(i,i)\, \b{D}(j,j)}}.
\end{align}
Note that Eq. (\ref{equation_Laplacian_matrix_W_over_D_D}) is very similar to one of the techniques for kernel normalization \cite{ah2010normalized}. Also, we know that adjacency matrix is a kernel, such as RBF kernel in Eq. (\ref{equation_RBF_kernel}). Hence, in the literature, sometimes Laplacian is referred to as the \textit{normalized kernel} matrix (e.g., see the literature of diffusion map \cite{coifman2006diffusion,nadler2006diffusion2}). 

Note that Eq. (\ref{equation_Laplacian_matrix_2}) has an opposite effect compared to Eq. (\ref{equation_Laplacian_matrix}) because the matrix $\b{W}$ is not negated in it. 
According to the spectral graph theory \cite{chung1997spectral}, we can use $(\b{I} - \b{L})$ or Eqs. (\ref{equation_Laplacian_matrix_D_W_D}) and (\ref{equation_Laplacian_matrix_W_over_D_D}) rather than $\b{L}$ in optimization problems if we change minimization to maximization.

It is also good to note that Eq. (\ref{equation_Laplacian_matrix_2}) is very similar to the \textit{symmetric normalized Laplacian matrix} defined as {\citep[p. 2]{chung1997spectral}}:
\begin{align}
\b{L} \gets \b{D}^{-(1/2)} \b{L} \b{D}^{-(1/2)}.
\end{align}

\subsection{Interpretation of Laplacian}

% google: laplacian operator vs laplacian matrix
% https://math.stackexchange.com/questions/2110237/what-is-the-relation-between-the-laplacian-operator-and-the-laplacian-matrix
% https://math.stackexchange.com/questions/1657788/discrete-laplacian

Let $\b{B}$ denote the incidence matrix of the data graph where the rows and columns of $\b{B}$ correspond to edges and vertices of the graph, respectively. 
The elements of this incidence matrix determine if a vertex is connected to an edge (it is zero if they are not connected). Therefore, this matrix shows if every two points (vertices) are connected or not. 
% The element $(i,j)$ of the incidence matrix is the weight of the edge connecting the points $\b{x}_i$ and $\b{x}_j$ if they are connected and is zero otherwise.

\begin{lemma}[{\citep[Section 5.6]{strang2014differential}}]\label{lemma_incidence_matrix}
The incidence matrix of a directed graph can be seen as a difference matrix (of graph nodes). 
\end{lemma}

\begin{corollary}
As gradient, denoted by $\nabla$, is a measure of small difference, we can say from Lemma \ref{lemma_incidence_matrix} that the incidence matrix is analogous to gradient:
\begin{align}\label{equation_incidence_and_gradient}
\b{B} \propto \nabla.
\end{align}
\end{corollary}

\begin{lemma}[{\citep[Lemma 10]{kelner2007an}}]
We can factorize the Laplacian matrix of a graph using its incidence matrix:
\begin{align}\label{equation_Laplacian_factorize_incidence}
\b{L} = \b{B}^\top \b{B}.
\end{align}
This also explains why the Laplacian matrix is positive semi-definite. 
\end{lemma}

\begin{corollary}
From Eqs. (\ref{equation_incidence_and_gradient}) and (\ref{equation_Laplacian_factorize_incidence}), we can say that:
\begin{align}\label{equation_Laplacian_factorize_gradient}
\b{L} = \nabla^\top \nabla.
\end{align}
This operator, which is the inner produce of gradients, is also referred to as the \textit{Laplace operator}.
\end{corollary}

Hence, the \textit{Laplace operator} or \textit{Laplacian} is a scalar operator which is the inner product of two gradient vectors:
\begin{equation}
\begin{aligned}
\b{L} &:= \nabla \cdot \nabla = \nabla^2 = \Big[\frac{\partial}{\partial x_1}, \dots, \frac{\partial}{\partial x_d}\Big]
\begin{bmatrix}
    \frac{\partial}{\partial x_1} \\
    \vdots \\
    \frac{\partial}{\partial x_d} 
\end{bmatrix} \\
&= \frac{\partial^2}{\partial x_1 \partial x_1} + \dots + \frac{\partial^2}{\partial x_d \partial x_d} = \sum_{i=1}^d \frac{\partial^2}{\partial x_i \partial x_i}.
\end{aligned}
\end{equation}
This shows that Laplacian is the summation of diagonal of the Hessian matrix (the Jacobian of the gradient). Therefore, we can say that Laplacian is the trace of the Hessian matrix. Thus, it can also be seen as the summation of eigenvalues of the Hessian matrix.
Laplacian captures the amount of divergence from a point.
Therefore, Laplacian measures whether the neighboring points of a point are larger or smaller than it or are on a line. In other words, it shows how much linear the neighborhood of a point is. If a point is a local minimum/maximum, its Laplacian is positive/negative because its diverging neighbors are upper/lower than it. In case, the point is on linear function, its Laplacian is zero.
Therefore, Laplacian shows how \textit{bumpy} the neighborhood of a point is\footnote{We thank Ming Miao for explaining this interpretation of Laplace operator to the authors.}.

% https://math.stackexchange.com/questions/1657788/discrete-laplacian
The above explanation can be stated in mathematics as divergence:
\begin{align}
\b{L} \overset{(\ref{equation_Laplacian_factorize_gradient})}{=} \nabla^\top \nabla = -\text{div} \nabla,
\end{align}
where div denotes the divergence operator. 

From the above discussion, we see that Laplacian measures how bumpy the neighbors of every point are with respect to it in the nearest neighbors graph of data. Therefore, using Laplacian in spectral clustering and related algorithms is because Laplacian is a measure of relation of neighbor points; hence, it can be used for inserting the information of structure of data into the formulation of algorithm. Similar to this interpretation, as well as some connections to heat equations, can be seen in \cite{belkin2005towards}.

\subsection{Eigenvalues of Laplacian Matrix}\label{section_eig_Laplacian}

It is well-known in linear algebra and graph theory that if a graph has $k$ disjoint connected parts, its Laplacian matrix has $k$ zero eigenvalues (see {\citep[Theorem 1]{anderson1985eigenvalues}}, {\citep[Theorem 3.10]{marsden2013eigenvalues}}, and \cite{polito2002grouping,ahmadizadeh2017eigenvalues}).
We assume that the nearest neighbors graph of data is a connected graph, $\b{L}$ has one zero eigenvalue whose eigenvector is $\b{1} = [1, 1, \dots, 1]^\top$. In spectral clustering and related methods, we ignore the eigenvector corresponding to the zero eigenvalue. 

\subsection{Convergence of Laplacian}

The Laplacian-based methods usually use either the Laplacian matrix in the Euclidean space or the Laplace--Beltrami operator on the Riemannian manifold \cite{hein2005graphs}. 
It is shown in several works, such as \cite{burago2015graph}, \cite{trillos2020error}, and \cite{dunson2021spectral}, that the eigenfunctions and eigenvalues of the Laplace--Beltrami operator can be approximated by eigenvalues and eigenvectors of Laplacian of a weight matrix, respectively. 
Moreover, it is shown in \cite{hein2005graphs,hein2007graph} that the kernel-based averaging operators can approximate the weighted Laplacian point-wise, for epsilon-graph (i.e., not fully-connected) weight matrices. 

\section{Spectral Clustering}\label{section_spectral_clustering}

Spectral clustering \cite{weiss1999segmentation,ng2001spectral} is a clustering method whose solution is reduced to an eigenvalue problem. Refer to \cite{weiss1999segmentation} for a survey on initial work on spectral clustering approaches. A list of some initial works developing to result in the modern spectral clustering is \cite{scott1990feature,costeira1995multi,perona1998factorization,shi1997normalized,shi2000normalized}.
The general idea of spectral clustering is as follows. First we transform data from its input space to another low-dimensional subspace so that the clusters of data become more separated hopefully. Then, we apply a simple clustering algorithm in that discriminating subspace. In the following, we explain the theory of spectral clustering. 

\subsection{Spectral Clustering}

\subsubsection{Adjacency Matrix}

As was explained in Section \ref{section_interpretation_Laplacian}, we compute the adjacency matrix and its Laplacian. These matrices are used in later sections for the spectral clustering algorithm. 

\subsubsection{The Cut}

We denote the graph of dataset by:
\begin{align}
\mathcal{G} := (\mathcal{X}, \mathcal{W}),
\end{align}
where the data indices $\mathcal{X} := \{i\}_{i=1}^n$ are the vertices of graph and the weights $\mathcal{W} := \{w_{ij}\}_{i,j=1}^n$ are the weights for edges of graph. 

Let the data points be divided into two disjoint clusters $\mathcal{A}$ and $\mathcal{A}'$:
\begin{align}
\mathcal{X} = \mathcal{A} \cup \mathcal{A}', \quad \mathcal{A} \cap \mathcal{A}' = \varnothing.
\end{align}
We define the cut between two sub-graphs as \cite{shi1997normalized,shi2000normalized}:
\begin{align}\label{equation_cut}
\text{cut}(\mathcal{A}, \mathcal{A}') := \sum_{i \in \mathcal{A}} \sum_{j \in \mathcal{A}'} w_{ij}.
\end{align}
As the summations can be exchanged, we have:
\begin{align}\label{equation_cut_symmetry}
\text{cut}(\mathcal{A}, \mathcal{A}') = \text{cut}(\mathcal{A}', \mathcal{A}).
\end{align}

\subsubsection{Optimization of Spectral Clustering}

If the two clusters or the sub-graphs are very different or far from each other, the summation of weights between their points is small so their cut is small. Hence, for finding the separated clusters, we can find the sub-graphs which minimize the cut:
\begin{align}
\underset{\mathcal{A}, \mathcal{A}'}{\text{minimize}} \quad \text{cut}(\mathcal{A}, \mathcal{A}') = \underset{\mathcal{A}, \mathcal{A}'}{\text{minimize}} \quad \sum_{i \in \mathcal{A}} \sum_{j \in \mathcal{A}'} w_{ij}.
\end{align}
However, this minimization is sensitive to outliers. This is because an outlier point can be considered as a cluster itself in this optimization problem as it is very far from other points. 
To overcome this issue, we use a normalized version of cut, named ratio cut, which is defined as \cite{shi1997normalized,shi2000normalized}:
\begin{align}
\text{ratioCut}(\mathcal{A}, \mathcal{A}') &= \frac{\text{cut}(\mathcal{A}, \mathcal{A}')}{|\mathcal{A}|} + \frac{\text{cut}(\mathcal{A}', \mathcal{A})}{|\mathcal{A}'|} \nonumber \\
&\overset{(\ref{equation_cut_symmetry})}{=} \frac{\text{cut}(\mathcal{A}, \mathcal{A}')}{|\mathcal{A}|} + \frac{\text{cut}(\mathcal{A}, \mathcal{A}')}{|\mathcal{A}'|}, \label{equation_ratio_cut}
\end{align}
where $|\mathcal{A}|$ denotes the cardinality of set $\mathcal{A}$. Minimizing the ratio cut rather than cut resolves the problem with outliers because the cluster of outlier(s) is a small cluster and its cardinality is small' hence, the denominator of one of the fractions in ratio cut gets very small and the ratio cut becomes large. Therefore, we minimize the ratio cut for clustering:
\begin{align}\label{equation_min_ratio_cut}
\underset{\mathcal{A}, \mathcal{A}'}{\text{minimize}} \quad \text{ratioCut}(\mathcal{A}, \mathcal{A}').
\end{align}

\begin{proposition}[\cite{shi1997normalized}]
Minimizing the ratio cut, i.e. Eq. (\ref{equation_min_ratio_cut}), is equivalent to:
\begin{align}\label{equation_minimization_spectral_clustering_w_f}
\underset{\b{f}}{\text{minimize}} \quad \sum_{i=1}^n \sum_{j=1}^n w_{ij}\, (f_i - f_j)^2, 
\end{align}
where $\b{f} := [f_1, \dots, f_n]^\top \in \mathbb{R}^n$ and:
\begin{align}\label{equation_f_definition}
f_i := 
\left\{
    \begin{array}{ll}
        \sqrt{\frac{|\mathcal{A}'|}{|\mathcal{A}|}} & \mbox{if } i \in \mathcal{A}, \\
        -\sqrt{\frac{|\mathcal{A}|}{|\mathcal{A}'|}} & \mbox{if } i \in \mathcal{A}'.
    \end{array}
\right.
\end{align}
\end{proposition}

\begin{proof}
\begin{align*}
&\sum_{i=1}^n \sum_{j=1}^n w_{ij}\, (f_i - f_j)^2 = \sum_{i \in \mathcal{A}} \sum_{j \in \mathcal{A}} w_{ij} \big(\underbrace{\sqrt{\frac{|\mathcal{A}'|}{|\mathcal{A}|}} - \sqrt{\frac{|\mathcal{A}'|}{|\mathcal{A}|}}}_{=0}\big)^2 \\
& + \sum_{i \in \mathcal{A}'} \sum_{j \in \mathcal{A}'} w_{ij} \big(\underbrace{-\sqrt{\frac{|\mathcal{A}|}{|\mathcal{A}'|}} + \sqrt{\frac{|\mathcal{A}|}{|\mathcal{A}'|}}}_{=0}\big)^2 \\
&+ \sum_{i \in \mathcal{A}} \sum_{j \in \mathcal{A}'} w_{ij} \big(\sqrt{\frac{|\mathcal{A}'|}{|\mathcal{A}|}} + \sqrt{\frac{|\mathcal{A}|}{|\mathcal{A}'|}}\big)^2 \\
& +\sum_{i \in \mathcal{A}'} \sum_{j \in \mathcal{A}} w_{ij} \big(-\sqrt{\frac{|\mathcal{A}|}{|\mathcal{A}'|}} - \sqrt{\frac{|\mathcal{A}'|}{|\mathcal{A}|}}\big)^2 \\
&= \sum_{i \in \mathcal{A}} \sum_{j \in \mathcal{A}'} w_{ij} \big(\sqrt{\frac{|\mathcal{A}'|}{|\mathcal{A}|}} + \sqrt{\frac{|\mathcal{A}|}{|\mathcal{A}'|}}\big)^2 \\
&+ \sum_{i \in \mathcal{A}'} \sum_{j \in \mathcal{A}} w_{ij} \big(\sqrt{\frac{|\mathcal{A}|}{|\mathcal{A}'|}} + \sqrt{\frac{|\mathcal{A}'|}{|\mathcal{A}|}}\big)^2 
\end{align*}
\begin{align*}
&= \sum_{i \in \mathcal{A}} \sum_{j \in \mathcal{A}'} w_{ij} \big(\frac{|\mathcal{A}'|}{|\mathcal{A}|} + \frac{|\mathcal{A}|}{|\mathcal{A}'|} + 2\big) \\
&+ \sum_{i \in \mathcal{A}'} \sum_{j \in \mathcal{A}} w_{ij} \big(\frac{|\mathcal{A}|}{|\mathcal{A}'|} + \frac{|\mathcal{A}'|}{|\mathcal{A}|} + 2\big) \\
&= \big(\frac{|\mathcal{A}|}{|\mathcal{A}'|} + \frac{|\mathcal{A}'|}{|\mathcal{A}|} + 2\big) \big( \sum_{i \in \mathcal{A}} \sum_{j \in \mathcal{A}'} w_{ij} + \sum_{i \in \mathcal{A}'} \sum_{j \in \mathcal{A}} w_{ij} \big) \\
&\overset{(\ref{equation_cut})}{=} \big(\frac{|\mathcal{A}|}{|\mathcal{A}'|} + \frac{|\mathcal{A}'|}{|\mathcal{A}|} + 2\big) \big( \text{cut}(\mathcal{A}, \mathcal{A}') + \text{cut}(\mathcal{A}', \mathcal{A}) \big) \\
&\overset{(\ref{equation_cut_symmetry})}{=} 2 \big(\frac{|\mathcal{A}|}{|\mathcal{A}'|} + \frac{|\mathcal{A}'|}{|\mathcal{A}|} + 2\big)\, \text{cut}(\mathcal{A}, \mathcal{A}')
\end{align*}
\begin{align*}
&= 2 \big(\frac{|\mathcal{A}|}{|\mathcal{A}'|} + \frac{|\mathcal{A}'|}{|\mathcal{A}|} + \underbrace{\frac{|\mathcal{A}|}{|\mathcal{A}|} + \frac{|\mathcal{A}'|}{|\mathcal{A}'|}}_{=2} \big)\, \text{cut}(\mathcal{A}, \mathcal{A}') \\
&= 2 \big( \frac{|\mathcal{A}| + |\mathcal{A}'|}{|\mathcal{A}'|} + \frac{|\mathcal{A}'| + |\mathcal{A}|}{|\mathcal{A}|} \big)\, \text{cut}(\mathcal{A}, \mathcal{A}') \\
&= 2 \big( \frac{n}{|\mathcal{A}'|} + \frac{n}{|\mathcal{A}|} \big)\, \text{cut}(\mathcal{A}, \mathcal{A}') \\
&= 2 n \big( \frac{\text{cut}(\mathcal{A}, \mathcal{A}')}{|\mathcal{A}'|} + \frac{\text{cut}(\mathcal{A}, \mathcal{A}')}{|\mathcal{A}|} \big) \overset{(\ref{equation_ratio_cut})}{=} 2n\, \text{ratioCut}(\mathcal{A}, \mathcal{A}').
\end{align*}
Q.E.D.
\end{proof}

\begin{proposition}[\cite{belkin2001laplacian}]\label{proposition_spectral_clustering_f_L_f}
We have:
\begin{align}
\frac{1}{2} \sum_{i=1}^n \sum_{j=1}^n w_{ij}\, (f_i - f_j)^2 = \b{f}^\top \b{L} \b{f},
\end{align}
where $\b{L} \in \mathbb{R}^{n \times n}$ is the Laplacian matrix of the weight matrix.
\end{proposition}
\begin{proof}
\begin{align*}
&\b{f}^\top \b{L} \b{f} \overset{(\ref{equation_Laplacian_matrix})}{=} \b{f}^\top (\b{D} - \b{W}) \b{f} = \b{f}^\top \b{D} \b{f} - \b{f}^\top \b{W} \b{f} \\
&\overset{(a)}{=} \sum_{i=1}^n d_i f_i^2 - \sum_{i=1}^n \sum_{j=1}^n w_{ij} f_i f_j \\
&= \frac{1}{2} (2\sum_{i=1}^n d_i f_i^2 - 2\sum_{i=1}^n \sum_{j=1}^n w_{ij} f_i f_j) \\
&\overset{(b)}{=} \frac{1}{2} (\sum_{i=1}^n d_i f_i^2 + \sum_{j=1}^n d_j f_j^2 - 2\sum_{i=1}^n \sum_{j=1}^n w_{ij} f_i f_j) \\
&\overset{(\ref{equation_degree_matrix})}{=} \frac{1}{2} (\sum_{i=1}^n f_i^2 \sum_{j=1}^n w_{ij} + \sum_{j=1}^n f_j^2 \sum_{i=1}^n w_{ji} \\
&~~~~~~~~~~~~~~~~~~~~~~~~~~~ - 2\sum_{i=1}^n \sum_{j=1}^n w_{ij} f_i f_j) 
\end{align*}
\begin{align*}
&\overset{(c)}{=} \frac{1}{2} (\sum_{i=1}^n \sum_{j=1}^n f_i^2 w_{ij} + \sum_{i=1}^n \sum_{j=1}^n f_j^2 w_{ij} \\
&~~~~~~~~~~~~~~~~~~~~~~~~~~~ - 2\sum_{i=1}^n \sum_{j=1}^n w_{ij} f_i f_j) \\
&= \frac{1}{2} \sum_{i=1}^n \sum_{j=1}^n w_{ij} (f_i - f_j)^2,
\end{align*}
where $(a)$ notices that $\b{D}$ is a diagonal matrix, $(b)$ is because we can replace the dummy index $i$ with $j$, and $(c)$ notices that $w_{ij} = w_{ji}$ as we usually use symmetric similarity measures. Q.E.D.
\end{proof}

Therefore, we can minimize $\b{f}^\top \b{L} \b{f}$ for spectral clustering. Hence, the optimization in spectral clustering is:
\begin{equation}
\begin{aligned}
& \underset{\b{f}}{\text{minimize}}
& & \b{f}^\top \b{L} \b{f} \\
& \text{subject to}
& & 
f_i := 
\left\{
    \begin{array}{ll}
        \sqrt{\frac{|\mathcal{A}'|}{|\mathcal{A}|}} & \mbox{if } i \in \mathcal{A}, \\
        -\sqrt{\frac{|\mathcal{A}|}{|\mathcal{A}'|}} & \mbox{if } i \in \mathcal{A}'.
    \end{array}
\right.
, \; \forall i.
\end{aligned}
\end{equation}
This optimization problem is very complicated to solve. We can relax the constraint to having orthonormal vector $\b{f}$.
Recall that according to Eq. (\ref{equation_f_definition}) or the constraint, if $f_i > 0$, it is in the first cluster or $\mathcal{A}$ and if $f_i < 0$, it is in the second cluster or $\mathcal{A}'$. 
Here, in the modified constraint, we keep the sign characteristic of Eq. (\ref{equation_f_definition}) as:
\begin{align}\label{equation_f_definition_sign}
f_i := 
\left\{
    \begin{array}{ll}
        > 0 & \mbox{if } i \in \mathcal{A}, \\
        < 0 & \mbox{if } i \in \mathcal{A}'.
    \end{array}
\right.
\end{align}
This modification satisfies $\sum_{i=1}^n (f_i \times f_i) = 1$.
Hence, the optimization is relaxed to:
\begin{equation}\label{equation_spectral_clustering_optimization}
\begin{aligned}
& \underset{\b{f}}{\text{minimize}}
& & \b{f}^\top \b{L} \b{f} \\
& \text{subject to}
& & 
\b{f}^\top \b{f} = 1.
\end{aligned}
\end{equation}

\subsubsection{Solution to Optimization}

The Lagrangian of Eq. (\ref{equation_spectral_clustering_optimization}) is \cite{boyd2004convex}:
\begin{align*}
\mathcal{L} = \b{f}^\top \b{L} \b{f} - \lambda (\b{f}^\top \b{f} - 1),
\end{align*}
with $\lambda$ as the dual variable. Setting the derivative of Lagrangian to zero gives:
\begin{align}
\mathbb{R}^n \ni \frac{\partial \mathcal{L}}{\partial \b{f}} = 2 \b{L} \b{f} - 2 \lambda \b{f} \overset{\text{set}}{=} \b{0} \implies \b{L} \b{f} = \lambda \b{f},
\end{align}
which is the eigenvalue problem for the Laplacian matrix $\b{L}$ \cite{ghojogh2019eigenvalue}. 

As Eq. (\ref{equation_spectral_clustering_optimization}) is a minimization problem, the eigenvectors should be sorted from the corresponding smallest to largest eigenvalues.
Moreover, note that $\b{L}$ is the Laplacian matrix for the weights $\b{W}$.
Hence, according to Section \ref{section_eig_Laplacian}, after sorting the eigenvectors from smallest to largest eigenvalues, we ignore the first eigenvector having zero eigenvalue and take the smallest eigenvector of $\b{L}$ with non-zero eigenvalue as the solution $\b{f} \in \mathbb{R}^n$. 

According to Eq. (\ref{equation_f_definition_sign}), after finding $\b{f} = [f_1, \dots, f_n]^\top$, we determine the cluster of each point as:
\begin{align}
i := 
\left\{
    \begin{array}{ll}
        \in \mathcal{A} & \mbox{if } f_i > 0, \\
        \in \mathcal{A}' & \mbox{if } f_i < 0.
    \end{array}
\right.
\end{align}

\subsubsection{Extension of Spectral Clustering to Multiple Clusters}\label{section_spectral_clustering_multipleClusters}

If we have more than binary clusters in data, the optimization (\ref{equation_spectral_clustering_optimization}):
\begin{equation}\label{equation_spectral_clustering_optimization_multiple}
\begin{aligned}
& \underset{\b{F}}{\text{minimize}}
& & \textbf{tr}(\b{F}^\top \b{L} \b{F}) \\
& \text{subject to}
& & 
\b{F}^\top \b{F} = \b{I},
\end{aligned}
\end{equation}
where $\textbf{tr}(.)$ denotes the trace of matrix and $\b{I}$ is the identity matrix.
The Lagrangian is \cite{boyd2004convex}:
\begin{align*}
\mathcal{L} = \b{F}^\top \b{L} \b{F} - \textbf{tr}(\b{\Lambda}^\top (\b{F}^\top \b{F} - \b{I})),
\end{align*}
where $\b{\Lambda}$ is the diagonal matrix with the dual variables. Setting the derivative of Lagrangian to zero gives:
\begin{align}\label{equation_spectral_clustering_eig_problem}
\mathbb{R}^{n \times n} \ni \frac{\partial \mathcal{L}}{\partial \b{F}} = 2 \b{L} \b{F} - 2 \b{F} \b{\Lambda} \overset{\text{set}}{=} \b{0} \implies \b{L} \b{F} = \b{F} \b{\Lambda},
\end{align}
which is the eigenvalue problem for the Laplacian matrix $\b{L}$ where columns of $\b{F}$ are the eigenvectors \cite{ghojogh2019eigenvalue}. 
As the problem is minimization, the columns of $\b{F}$ are sorted from corresponding smallest to largest eigenvalues. Also, as it is eigenvalue problem for the Laplacian matrix, we ignore the smallest eigenvector with eigenvalue zero (see Section \ref{section_eig_Laplacian}). 

If we have $c$ number of clusters, we truncate the sorted $\b{F}$ to have the $c$ smallest eigenvectors with non-zero eigenvalues. Hence, we will have $\b{F} \in \mathbb{R}^{n \times c}$. 

After finding the matrix $\b{F} \in \mathbb{R}^{n \times c}$, we treat the rows of $\b{F}$ as new data points in a $c$-dimensional space. 
We apply another clustering algorithm, such as K-means, on the rows of $\b{F}$ in order to cluster the $n$ points into $c$ clusters \cite{ng2001spectral}. 
Note that regular clustering methods, such as K-means, are applied on the input space (space of $\b{X}$) while spectral clustering is applied on the embedding space (space of $\b{Y}$). In other words, spectral clustering first extracts features for better discrimination of clusters and then applies clustering. This usually results in better clustering because the extracted features are supposed to discriminate the data clusters more. 

Recall that as was mentioned in Section \ref{section_Laplacian_matrix_definition}, some papers use Eq. (\ref{equation_Laplacian_matrix_D_W_D}) for Laplacian and they use maximization instead of minimization.

% It is noteworthy that according to the spectral graph theory \cite{chung1997spectral}, we can use $(\b{I} - \b{L})$ rather than $\b{L}$ in Eq. (\ref{equation_spectral_clustering_optimization_multiple}) if we change minimization to maximization. 
% As was mentioned by Eq. (\ref{equation_Laplacian_matrix_2}), the papers \cite{weiss1999segmentation,ng2001spectral} have used:
% \begin{align}\label{equation_Laplacian_matrix_D_W_D}
% \b{L} := \b{D}^{-1/2} \b{W} \b{D}^{-1/2},
% \end{align}
% rather than Eq. (\ref{equation_Laplacian_matrix}) and thus have used a maximization problem. In maximization, we should sort the eigenvectors from the corresponding largest to smallest eigenvalues. 
% According to \cite{weiss1999segmentation}, Eq. (\ref{equation_Laplacian_matrix_D_W_D}) leads to having the $(i,j)$-th element of Laplacian as:
% \begin{align}\label{equation_Laplacian_matrix_W_over_D_D}
% \b{L}(i,j) := \frac{\b{W}(i,j)}{\sqrt{\b{D}(i,i)\, \b{D}(j,j)}}.
% \end{align}
% Note that Eq. () is very similar to one of the techniques for kernel normalization \cite{ah2010normalized}. Also, we know that adjacency matrix is a kernel, such as RBF kernel in Eq. (\ref{equation_RBF_kernel}). Hence, in the literature, sometimes Laplacian is referred to as the normalized kernel matrix (e.g., see the literature of diffusion map \cite{coifman2006diffusion,nadler2006diffusion2}). 

\subsubsection{Optimization Approach 2}

In the original work of spectral clustering, the following optimization is used rather than Eq. (\ref{equation_spectral_clustering_optimization_multiple}) \cite{shi1997normalized}:
\begin{equation}\label{equation_spectral_clustering_optimization_multiple_2}
\begin{aligned}
& \underset{\b{F}}{\text{minimize}}
& & \textbf{tr}(\b{F}^\top \b{L} \b{F}) \\
& \text{subject to}
& & 
\b{F}^\top \b{D} \b{F} = \b{I},
\end{aligned}
\end{equation}
which makes the rotated $\b{F}$, by the degree matrix, orthogonal. Note that, according to Eq. (\ref{equation_degree_matrix}), the degree matrix is a natural measure for the data points in the way that the larger $d_i$ corresponds to the more important $f_i$ because it has more similarity (weights) with its neighbors \cite{he2004locality}. That is why having $\b{D}$ in the constraint makes sense as it gives more weight to the more important $f_i$'s.

The Lagrangian is \cite{boyd2004convex}:
\begin{align*}
\mathcal{L} = \b{F}^\top \b{L} \b{F} - \textbf{tr}(\b{\Lambda}^\top (\b{F}^\top \b{D} \b{F} - \b{I})),
\end{align*}
where $\b{\Lambda}$ is the diagonal matrix with the dual variables. Setting the derivative of Lagrangian to zero gives:
\begin{align}
\mathbb{R}^{n \times n} \ni \frac{\partial \mathcal{L}}{\partial \b{F}} = 2 \b{L} \b{F} - 2 \b{D} \b{F} \b{\Lambda} \overset{\text{set}}{=} \b{0} \implies \b{L} \b{F} = \b{D} \b{F} \b{\Lambda},
\end{align}
which is the generalized eigenvalue problem $(\b{L}, \b{D})$ where columns of $\b{F}$ are the eigenvectors \cite{ghojogh2019eigenvalue}. 
Sorting the matrix $\b{F}$, ignoring the eigenvector with zero eigenvalue (see Section \ref{section_eig_Laplacian}), and the rest of algorithm are done as explained before.

\subsection{Other Improvements over Spectral Clustering}

There exist many different improvements over the basic spectral clustering which was explained. Some of these developments are distributed spectral clustering \cite{chen2010parallel}, consistency spectral clustering \cite{von2008consistency}, correctional spectral clustering \cite{blaschko2008correlational}, spectral clustering by autoencoder \cite{banijamali2017fast}, multi-view spectral clustering \cite{kumar2011co,kumar2011co2,yin2019multi}, self-tuning spectral clustering \cite{zelnik2004self}, and fuzzy spectral clustering \cite{yang2016fuzzy}. 
Some existing surveys and tutorials on spectral clustering are  \cite{von2007tutorial,guan2008survey,nascimento2011spectral,guo2012survey}. The optimization set-up of spectral clustering has also been successfully used in the field of hashing. Spectral hashing \cite{weiss2008spectral} is a successful example. 
Moreover, spectral clustering has had various applications, e.g., in image segmentation \cite{yang2016novel}.

\section{Laplacian Eigenmap}\label{section_Laplacian_eigenmap}

\subsection{Laplacian Eigenmap}

Laplacian eigenmap \cite{belkin2001laplacian,belkin2003laplacian} is a \textit{nonlinear} dimensionality reduction method which is related to spectral clustering. 
As was explained in Section \ref{section_spectral_clustering_multipleClusters}, spectral clustering extracts features for better discriminative clusters and then applies clustering in the embedding space. Therefore, spectral clustering has some dimensionality reduction and feature extraction task inside it. This guided researchers to propose Laplacian eigenmap based on the concepts of spectral clustering. 

Consider the dataset matrix $\b{X} := [\b{x}_1, \dots, \b{x}_n] \in \mathbb{R}^{d \times n}$. We desire $p$-dimensional embeddings of data points $\b{Y} := [\b{y}_1, \dots, \b{y}_n] \in \mathbb{R}^{n \times p}$ where $p \leq d$ and usually $p \ll d$. Note that in Laplacian eigenmap, for simplicity of algebra, we put the embedding vectors row-wise in the matrix $\b{Y}$. 

\subsubsection{Adjacency Matrix}

Similar to spectral clustering, we create an adjacency matrix using Eq. (\ref{equation_adjacency_matrix}). The elements $w_{ij}$ can be determined by the RBF kernel, i.e. Eq. (\ref{equation_RBF_kernel}), or the simple-minded approach, i.e. Eq. (\ref{equation_simple_minded}).

\subsubsection{Interpretation of Laplacian Eigenmap}\label{section_LE_interpretation}

Inspired by Eq. (\ref{equation_minimization_spectral_clustering_w_f}) in spectral clustering, the optimization problem of Laplacian eigenmap is:
\begin{align}\label{equation_optimization_Laplacian_eigenmap_1}
\underset{\b{Y}}{\text{minimize}} \quad \sum_{i=1}^n \sum_{j=1}^n w_{ij}\, \|\b{y}_i - \b{y}_j\|_2^2.
\end{align}
Eq. (\ref{equation_optimization_Laplacian_eigenmap_1}) can be interpreted as the following \cite{belkin2003laplacian}:
\begin{itemize}
\item If $\b{x}_i$ and $\b{x}_j$ are close to each other, $w_{ij}$ is large. Hence, for minimizing the objective function, the term $\|\b{y}_i - \b{y}_j\|_2^2$ should be minimized which results in close $\b{y}_i$ and $\b{y}_j$. This is expected because $\b{x}_i$ and $\b{x}_j$ are close and their embeddings should be close as well. 
\item If $\b{x}_i$ and $\b{x}_j$ are far apart, $w_{ij}$ is small. Hence, the objective function is small because of being multiplied by the small weight $w_{ij}$. Hence, we do not care about $\b{y}_i$ and $\b{y}_j$ much as the objective is already small. 
\end{itemize}
According to the above two notes, we see that Laplacian eigenmap fits data locally because it cares more about the nearby points and hopes that the global structure is preserved by this local fitting \cite{saul2003think}. In other words, we only care about nearby points and hope for correct relative embedding of far points. This might be considered as a weakness of Laplacian eigenmap. In another word, Laplacian eigenmap preserves locality of data pattern. This connects this method to the locality preserving projection which is explained later in Section \ref{section_Locality_Preserving_Projection}. 

\subsubsection{Optimization Approach 1}

According to Proposition \ref{proposition_spectral_clustering_f_L_f}, we restate Eq. (\ref{equation_optimization_Laplacian_eigenmap_1}) with adding a constraint inspired by Eq. (\ref{equation_spectral_clustering_optimization_multiple}) as:
\begin{equation}\label{equation_optimization_Laplacian_eigenmap_2}
\begin{aligned}
& \underset{\b{Y}}{\text{minimize}}
& & \textbf{tr}(\b{Y}^\top \b{L} \b{Y}) \\
& \text{subject to}
& & 
\b{Y}^\top \b{Y} = \b{I},
\end{aligned}
\end{equation}
where $\b{L}$ is the Laplacian of the adjacency matrix. 

The Lagrangian for this optimization problem is \cite{boyd2004convex}:
\begin{align*}
\mathcal{L} = \b{Y}^\top \b{L} \b{Y} - \textbf{tr}(\b{\Lambda}^\top (\b{Y}^\top \b{Y} - \b{I})),
\end{align*}
where $\b{\Lambda}$ is the diagonal matrix with the dual variables. Setting the derivative of Lagrangian to zero gives:
\begin{align}
\mathbb{R}^{n \times n} \ni \frac{\partial \mathcal{L}}{\partial \b{Y}} = 2 \b{L} \b{Y} - 2 \b{Y} \b{\Lambda} \overset{\text{set}}{=} \b{0} \implies \b{L} \b{Y} = \b{Y} \b{\Lambda},
\end{align}
which is the eigenvalue problem for the Laplacian matrix $\b{L}$ where columns of $\b{Y}$ are the eigenvectors \cite{ghojogh2019eigenvalue}. 
As the problem is minimization, the columns of $\b{Y}$ are sorted from corresponding smallest to largest eigenvalues. Also, as it is eigenvalue problem for the Laplacian matrix, we ignore the smallest eigenvector with eigenvalue zero (see Section \ref{section_eig_Laplacian}). 
For a $p$-dimensional embedding space, we truncate the sorted $\b{Y}$ to have the $p$ smallest eigenvectors with non-zero eigenvalues. Hence, we will have $\b{Y} \in \mathbb{R}^{n \times p}$. The rows of $\b{Y}$ are the $p$-dimensional embedding vectors.

Note that as the adjacency matrix can use any kernel, such as the RBF kernel of Eq. (\ref{equation_RBF_kernel}), some papers have named it the \textit{kernel Laplacian eigenmap} \cite{guo2006kernel}. 

\subsubsection{Optimization Approach 2}\label{section_Laplacian_eigenmap_optimization_approach_2}

Another optimization approach for Laplacian eigenmap is to use the following optimization rather than Eq. (\ref{equation_optimization_Laplacian_eigenmap_2}) \cite{belkin2001laplacian}:
\begin{equation}\label{equation_optimization_Laplacian_eigenmap_3}
\begin{aligned}
& \underset{\b{Y}}{\text{minimize}}
& & \textbf{tr}(\b{Y}^\top \b{L} \b{Y}) \\
& \text{subject to}
& & 
\b{Y}^\top \b{D} \b{Y} = \b{I},
\end{aligned}
\end{equation}
where $\b{D}$ is the degree matrix defined by Eq. (\ref{equation_degree_matrix}). In this optimization, rather than having orthonormal embeddings, we want the rotated embeddings by the degree matrix to be orthonormal. 
Note that, according to Eq. (\ref{equation_degree_matrix}), the degree matrix is a natural measure for the data points in the way that the larger $d_i$ corresponds to the more important $\b{x}_i$ because it has more similarity (weights) with its neighbors \cite{he2004locality}. That is why having $\b{D}$ in the constraint makes sense as it gives more weight to the more important $\b{y}_i$'s.

The Lagrangian for this optimization problem is \cite{boyd2004convex}:
\begin{align*}
\mathcal{L} = \b{Y}^\top \b{L} \b{Y} - \textbf{tr}(\b{\Lambda}^\top (\b{Y}^\top \b{D} \b{Y} - \b{I})),
\end{align*}
where $\b{\Lambda}$ is the diagonal matrix with the dual variables. Setting the derivative of Lagrangian to zero gives:
\begin{align}
&\mathbb{R}^{n \times n} \ni \frac{\partial \mathcal{L}}{\partial \b{Y}} = 2 \b{L} \b{Y} - 2 \b{D} \b{Y} \b{\Lambda} \overset{\text{set}}{=} \b{0} \nonumber \\
&~~~~~~~~~~~ \implies \b{L} \b{Y} = \b{D} \b{Y} \b{\Lambda}, \label{equation_Laplacian_eigenmap_generalized_eig_problem}
\end{align}
which is the generalized eigenvalue problem $(\b{L}, \b{D})$ where columns of $\b{Y}$ are the eigenvectors \cite{ghojogh2019eigenvalue}. 
Sorting the matrix $\b{Y}$ and ignoring the eigenvector with zero eigenvalue (see Section \ref{section_eig_Laplacian}) is done as explained before. 

It is shown in {\citep[Normalization Lemma 1]{weiss1999segmentation}} that the solution to Eq. (\ref{equation_Laplacian_eigenmap_generalized_eig_problem}) is equivalent to the solution of eigenvalue problem for Eq. (\ref{equation_Laplacian_matrix_W_over_D_D}). In other words, the eigenvectors of Eq. (\ref{equation_Laplacian_matrix_W_over_D_D}) are equivalent to generalized eigenvectors of Eq. (\ref{equation_Laplacian_eigenmap_generalized_eig_problem}).

\subsection{Out-of-sample Extension for Laplacian Eigenmap}

So far, we embedded the training dataset $\{\b{x}_i \in \mathbb{R}^d\}_{i=1}^n$ or $\b{X} = [\b{x}_1, \dots, \b{x}_n] \in \mathbb{R}^{d \times n}$ to have their embedding $\{\b{y}_i \in \mathbb{R}^p\}_{i=1}^n$ or $\b{Y} = [\b{y}_1, \dots, \b{y}_n]^\top \in \mathbb{R}^{n \times p}$.
Assume we have some out-of-sample (test data), denoted by $\{\b{x}_i^{(t)} \in \mathbb{R}^d\}_{i=1}^{n_t}$ or $\b{X}_t = [\b{x}_1^{(t)}, \dots, \b{x}_n^{(t)}] \in \mathbb{R}^{d \times n_t}$. We want to find their embedding $\{\b{y}_i^{(t)} \in \mathbb{R}^p\}_{i=1}^{n_t}$ or $\b{Y}_t = [\b{y}_1^{(t)}, \dots, \b{y}_n^{(t)}]^\top \in \mathbb{R}^{n_t \times p}$ after the training phase. 

\subsubsection{Eigenfunctions}

Consider a Hilbert space $\mathcal{H}_p$ of functions with the inner product $\langle f,g \rangle = \int f(x) g(x) p(x) dx$ with density function $p(x)$. In this space, we can consider the kernel function $K_p$:
\begin{align}
(K_p f)(x) = \int K(x,y)\, f(y)\, p(y)\, dy,
\end{align}
where the density function can be approximated empirically. 
The \textit{eigenfunction decomposition} is defined to be \cite{bengio2004learning,bengio2004out}:
\begin{align}
(K_p f_k)(x) = \delta'_k f_k(x),
\end{align}
where $f_k(x)$ is the $k$-th \textit{eigenfunction} and $\delta'_k$ is the corresponding eigenvalue.
If we have the eigenvalue decomposition \cite{ghojogh2019eigenvalue} for the kernel matrix $\b{K}$, we have $\b{K} \b{v}_k = \delta_k \b{v}_k$ where $\b{v}_k$ is the $k$-th eigenvector and $\delta_k$ is the corresponding eigenvalue. According to {\citep[Proposition 1]{bengio2004out}}, we have $\delta'_k = (1/n) \delta_k$. 

\subsubsection{Embedding Using Eigenfunctions}

\begin{proposition}
If $v_{ki}$ is the $i$-th element of the $n$-dimensional vector $\b{v}_k$ and $k(\b{x}, \b{x}_i)$ is the kernel between vectors $\b{x}$ and $\b{x}_i$, the eigenfunction for the point $\b{x}$ and the $i$-th training point $\b{x}_i$ are:
\begin{align}
f_k(\b{x}) &= \frac{\sqrt{n}}{\delta_k} \sum_{i=1}^n v_{ki}\, \breve{k}_t(\b{x}_i, \b{x}), \\
f_k(\b{x}_i) &= \sqrt{n}\, v_{ki}, 
\end{align}
respectively, where $\breve{k}_t(\b{x}_i, \b{x})$ is the centered kernel between training set and the out-of-sample point $\b{x}$. 

Let the embedding of the point $\b{x}$ be $\mathbb{R}^p \ni \b{y}(\b{x}) = [y_1(\b{x}), \dots, y_p(\b{x})]^\top$. The $k$-th dimension of this embedding is:
\begin{align}\label{equation_embedding_eigenfunction}
y_k(\b{x}) &= \sqrt{\delta_k}\, \frac{f_k(\b{x})}{\sqrt{n}} = \frac{1}{\sqrt{\delta_k}} \sum_{i=1}^n v_{ki}\, \breve{k}_t(\b{x}_i, \b{x}).
\end{align}
\end{proposition}

\begin{proof}
This proposition is taken from {\citep[Proposition 1]{bengio2004out}}. For proof, refer to {\citep[Proposition 1]{bengio2004learning}}, {\citep[Proposition 1]{bengio2006spectral}}, and {\citep[Proposition 1 and Theorem 1]{bengio2003spectral}}. 
More complete proofs can be found in \cite{bengio2003learning}. 
\end{proof}

We can see the adjacency matrix as a kernel (e.g. see Eq. (\ref{equation_RBF_kernel}) where RBF kernel is used for adjacency matrix). This fact was also explained in Section \ref{section_Laplacian_matrix_definition} where we said why Laplacian can be seen as a normalized kernel matrix. Hence, here, kernel $k_t(\b{x}_i, \b{x})$ refers to the adjacency matrix. 

According to Eq. (\ref{equation_Laplacian_matrix_W_over_D_D}) and Eq. (\ref{equation_spectral_clustering_eig_problem}), the solution of spectral clustering is the eigenvectors of Laplacian in Eq. (\ref{equation_Laplacian_matrix_W_over_D_D}). Moreover, as was explained in Section \ref{section_Laplacian_eigenmap_optimization_approach_2}, the solution of Laplacian eigenmap is also the eigenvectors of Laplacian in Eq. (\ref{equation_Laplacian_matrix_W_over_D_D}). Therefore, eigenfunctions can be used for out-of-sample extension. 

If we have a set of $n_t$ out-of-sample data points, we can compute the kernel $\b{K}_t \in \mathbb{R}^{n \times n_t}$ between training and out-of-sample data. 
We can normalize the kernel matrix \cite{ah2010normalized} as follows. 
The $\breve{k}_t(\b{x}_i, \b{x})$, used in Eq. (\ref{equation_embedding_eigenfunction}), is an element of the normalized kernel (adjacency) matrix \cite{bengio2003out}:
\begin{align}\label{equation_normalized_outOfSample_kernel}
\breve{k}_t(\b{x}_i, \b{x}) := \frac{1}{n} \frac{k_t(\b{x}_i, \b{x})}{\sqrt{\mathbb{E}[\b{K}_t(\b{x}_i, .)]\,\, \mathbb{E}[\b{K}_t(.,\b{x})]}},
\end{align}
where the expectations are computed as:
\begin{align}
&\mathbb{E}[\b{K}_t(\b{x}_i, .)] \approx \frac{1}{n_t} \sum_{j=1}^{n_t} k_t(\b{x}_i, \b{x}_j^{(t)}), \label{equation_kernel_expectation_columns} \\
&\mathbb{E}[\b{K}_t(., \b{x})] \approx \frac{1}{n} \sum_{i=1}^{n} k_t(\b{x}_i, \b{x}). \label{equation_kernel_expectation_rows}
\end{align}
Noticing the Eq. (\ref{equation_degree_matrix}) for $\b{D}$ and comparing Eq. (\ref{equation_Laplacian_matrix_W_over_D_D}) to Eqs. (\ref{equation_normalized_outOfSample_kernel}), (\ref{equation_kernel_expectation_columns}), and (\ref{equation_kernel_expectation_rows}) show the approximate equivalency of Eqs. (\ref{equation_Laplacian_matrix_W_over_D_D}) and (\ref{equation_normalized_outOfSample_kernel}), particularly if we assume $k_t \approx k$ and so $n_t \approx n$. Hence, Eq. (\ref{equation_normalized_outOfSample_kernel}) is valid to be used in Eq. (\ref{equation_embedding_eigenfunction}). 

\subsubsection{Out-of-sample Embedding}

One can use Eq. (\ref{equation_embedding_eigenfunction}) to embed the $i$-th out-of-sample data point $\b{x}_i^{(t)}$. For this purpose, $\b{x}_i^{(t)}$ should be used in place of $\b{x}$ in Eq. (\ref{equation_embedding_eigenfunction}). This out-of-sample extension technique can be used for both Laplacian eigenmap and spectral clustering \cite{bengio2004out}. 
In addition to this technique, there exist some other methods for out-of-sample extension of Laplacian eigenmap which we pass by in this paper. For example, two other extension methods are kernel mapping \cite{gisbrecht2012out,gisbrecht2015parametric} and through optimization \cite{bunte2012general}.

\subsection{Other Improvements over Laplacian Eigenmap}

There have been recent developments over the basic Laplacian eigenmap. Some examples are Laplacian Eigenmaps Latent Variable Model (LELVM) \cite{carreira2007laplacian}, robust Laplacian eigenmap \cite{roychowdhury2009robust}, and Laplacian forest \cite{lombaert2014laplacian}.
Convergence analysis of Laplacian eigenmap methods can be found in several papers \cite{belkin2006convergence,singer2006graph}. 
Inspired by eigenfaces \cite{turk1991eigenfaces} and Fisherfaces \cite{belhumeur1997eigenfaces} which have been proposed based on principal component analysis \cite{ghojogh2019unsupervised} and Fisher discriminant analysis \cite{ghojogh2019fisher}, respectively, \textit{Laplacianfaces} \cite{he2005face} has been proposed for face recognition using Laplacian eigenmaps. 
Finally, a recent survey on Laplacian eigenmap is \cite{li2019survey,wiskott2019laplacian}.

\section{Locality Preserving Projection}\label{section_Locality_Preserving_Projection}

Locality Preserving Projection (LPP) \cite{he2004locality} is a \textit{linear} dimensionality reduction method. It is a linear approximation of Laplacian eigenmap \cite{belkin2001laplacian,belkin2003laplacian}. It uses linear projection in the formulation of Laplacian eigenmap. Hence, in the following, we first introduce the concept of linear projection. 
Note that this method is called locality preserving projection because it approximates the Laplacian eigenmap by linear projection where, as was explained in Section \ref{section_LE_interpretation}, the Laplacian eigenmap fits data locally and preserves the local structure of data in the embedding space.

\subsection{Linear Projection}

\subsubsection{Projection Formulation}

Assume we have a data point $\b{x} \in \mathbb{R}^d$. We want to project this data point onto the vector space spanned by $p$ vectors $\{\b{u}_1, \dots, \b{u}_p\}$ where each vector is $d$-dimensional and usually $p \ll d$. We stack these vectors column-wise in matrix $\b{U} = [\b{u}_1, \dots, \b{u}_p] \in \mathbb{R}^{d \times p}$. In other words, we want to project $\b{x}$ onto the column space of $\b{U}$, denoted by $\mathbb{C}\text{ol}(\b{U})$.

The projection of $\b{x} \in \mathbb{R}^d$ onto $\mathbb{C}\text{ol}(\b{U}) \in \mathbb{R}^p$ and then its representation in the $\mathbb{R}^d$ (its reconstruction) can be seen as a linear system of equations:
\begin{align}\label{equation_projection}
\mathbb{R}^d \ni \widehat{\b{x}} := \b{U \beta},
\end{align}
where we should find the unknown coefficients $\b{\beta} \in \mathbb{R}^p$. 

If the $\b{x}$ lies in the $\mathbb{C}\text{ol}(\b{U})$ or $\textbf{span}\{\b{u}_1, \dots, \b{u}_p\}$, this linear system has exact solution, so $\widehat{\b{x}} = \b{x} = \b{U \beta}$. However, if $\b{x}$ does not lie in this space, there is no any solution $\b{\beta}$ for $\b{x} = \b{U \beta}$ and we should solve for projection of $\b{x}$ onto $\mathbb{C}\text{ol}(\b{U})$ or $\textbf{span}\{\b{u}_1, \dots, \b{u}_p\}$ and then its reconstruction. In other words, we should solve for Eq. (\ref{equation_projection}). In this case, $\widehat{\b{x}}$ and $\b{x}$ are different and we have a residual:
\begin{align}\label{equation_residual_1}
\b{r} = \b{x} - \widehat{\b{x}} = \b{x} - \b{U \beta},
\end{align}
which we want to be small. As can be seen in Fig. \ref{figure_residual_and_space}, the smallest residual vector is orthogonal to $\mathbb{C}\text{ol}(\b{U})$; therefore:
\begin{align}
\b{x} - \b{U\beta} \perp \b{U} &\implies \b{U}^\top (\b{x} - \b{U \beta}) = 0, \nonumber \\
& \implies \b{\beta} = (\b{U}^\top \b{U})^{-1} \b{U}^\top \b{x}. \label{equation_beta}
\end{align}
It is noteworthy that the Eq. (\ref{equation_beta}) is also the formula of coefficients in linear regression \cite{friedman2001elements} where the input data are the rows of $\b{U}$ and the labels are $\b{x}$; however, our goal here is different. 

\begin{figure}[!t]
\centering
\includegraphics[width=2.2in]{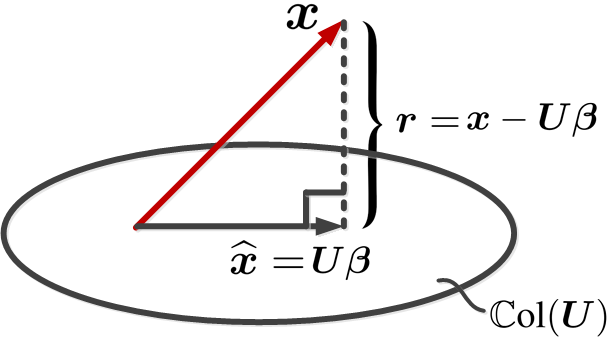}
\caption{The residual and projection onto the column space of $\b{U}$. Credit of image is for \cite{ghojogh2019unsupervised}.}
\label{figure_residual_and_space}
\end{figure}

Plugging Eq. (\ref{equation_beta}) in Eq. (\ref{equation_projection}) gives us:
\begin{align*}
\widehat{\b{x}} = \b{U} (\b{U}^\top \b{U})^{-1} \b{U}^\top \b{x}.
\end{align*}
We define:
\begin{align}\label{equation_hat_matrix}
\mathbb{R}^{d \times d} \ni \b{\Pi} := \b{U} (\b{U}^\top \b{U})^{-1} \b{U}^\top,
\end{align}
as ``projection matrix'' because it projects $\b{x}$ onto $\mathbb{C}\text{ol}(\b{U})$ (and reconstructs back).
Note that $\b{\Pi}$ is also referred to as the ``hat matrix'' in the literature because it puts a hat on top of $\b{x}$.

If the vectors $\{\b{u}_1, \dots, \b{u}_p\}$ are orthonormal (the matrix $\b{U}$ is orthogonal), we have $\b{U}^\top = \b{U}^{-1}$ and thus $\b{U}^\top \b{U} = \b{I}$. Therefore, Eq. (\ref{equation_hat_matrix}) is simplified:
\begin{align}
& \b{\Pi} = \b{U} \b{U}^\top.
\end{align}
So, we have:
\begin{align}\label{equation_x_hat}
\widehat{\b{x}} = \b{\Pi}\, \b{x} = \b{U} \b{U}^\top \b{x}.
\end{align}

\subsubsection{Projection onto a Subspace}\label{section_projection_subspace}

In subspace learning, the projection of a vector $\b{x} \in \mathbb{R}^d$ onto the column space of $\b{U} \in \mathbb{R}^{d \times p}$ (a $p$-dimensional subspace spanned by $\{\b{u}_j\}_{j=1}^p$ where $\b{u}_j \in \mathbb{R}^d$) is defined as:
\begin{align}
&\mathbb{R}^{p} \ni \widetilde{\b{x}} := \b{U}^\top \b{x}, \label{equation_projection_training_onePoint_severalDirections} \\
&\mathbb{R}^{d} \ni \widehat{\b{x}} := \b{U}\b{U}^\top \b{x} = \b{U} \widetilde{\b{x}}, \label{equation_reconstruction_training_onePoint_severalDirections}
\end{align}
where $\widetilde{\b{x}}$ and $\widehat{\b{x}}$ denote the projection and reconstruction of $\b{x}$, respectively.

If we have $n$ data points, $\{\b{x}_i\}_{i=1}^n$, which can be stored column-wise in a matrix $\b{X} \in \mathbb{R}^{d \times n}$, the projection and reconstruction of $\b{X}$ are defined as:
\begin{align}
&\mathbb{R}^{p \times n} \ni \widetilde{\b{X}} := \b{U}^\top \b{X}, \label{equation_projection_training_SeveralPoints_severalDirections} \\
&\mathbb{R}^{d \times n} \ni \widehat{\b{X}} := \b{U}\b{U}^\top \b{X} = \b{U} \widetilde{\b{X}}, \label{equation_reconstruction_training_SeveralPoints_severalDirections}
\end{align}
respectively.

If we have an out-of-sample data point $\b{x}_t$ which was not used in calculation of $\b{U}$, the projection and reconstruction of it are defined as:
\begin{align}
&\mathbb{R}^{p} \ni \widetilde{\b{x}}_t := \b{U}^\top \b{x}_t, \label{equation_projection_outOfSample_onePoint_severalDirections} \\
&\mathbb{R}^{d} \ni \widehat{\b{x}}_t := \b{U}\b{U}^\top \b{x}_t = \b{U} \widetilde{\b{x}}_t, \label{equation_reconstruction_outOfSample_onePoint_severalDirections}
\end{align}
respectively.

In case we have $n_t$ out-of-sample data points, $\{\b{x}_{t,i}\}_{i=1}^{n_t}$, which can be stored column-wise in a matrix $\b{X}_t \in \mathbb{R}^{d \times n_t}$, the projection and reconstruction of $\b{X}_t$ are defined as:
\begin{align}
&\mathbb{R}^{p \times n_t} \ni \widetilde{\b{X}}_t := \b{U}^\top \b{X}_t, \label{equation_projection_outOfSample_SeveralPoints_severalDirections} \\
&\mathbb{R}^{d \times n_t} \ni \widehat{\b{X}}_t := \b{U}\b{U}^\top \b{X}_t = \b{U} \widetilde{\b{X}}_t, \label{equation_reconstruction_outOfSample_SeveralPoints_severalDirections}
\end{align}
respectively.

For the properties of the projection matrix $\b{U}$, refer to \cite{ghojogh2019unsupervised}.

\subsection{Locality Preserving Projection}

LPP assumes that the embeddings of data points are obtained by a linear projection. It approximates Laplacian eigenmap linearly. We want to find the embeddings of data points $\b{X} = [\b{x}_1, \dots, \b{x}_n] \in \mathbb{R}^{d \times n}$ using linear projection. 

\subsubsection{One-dimensional Subspace}

For derivation of LPP optimization, we assume $p=1$ so the projection is onto a line with projection vector $\b{u} \in \mathbb{R}^d$. Later, we will extend to higher values for $p$. In this case, the vector of projections of all points is:
\begin{align}\label{equation_LPP_projection_to_line}
\mathbb{R}^{1 \times n} \ni \b{Y} = [y_1, y_2, \dots, y_n] := \b{u}^\top \b{X}.
\end{align}
Hence, the one-dimensional embedding of every point is:
\begin{align}\label{equation_LPP_projection_to_line_2}
\mathbb{R} \ni y_i = \b{u}^\top \b{x}_i, \quad \forall i \in \{1, \dots, n\}.
\end{align}

% LPP assumes that these embedding vectors are found by linear projection. As we discussed in Section \ref{section_projection_subspace}, projected data $\widetilde{\b{X}}$ are stored column-wise. Hence, we have:
% \begin{align}
% &\mathbb{R}^{p} \ni \b{y} = \widetilde{\b{x}} \overset{(\ref{equation_projection_training_SeveralPoints_severalDirections})}{=} \b{x}^\top \b{U}, \\
% &\mathbb{R}^{n \times p} \ni \b{Y} = \widetilde{\b{X}}^\top \overset{(\ref{equation_projection_training_SeveralPoints_severalDirections})}{=} \b{X}^\top \b{U}.
% \end{align}

Inspired by Eq. (\ref{equation_optimization_Laplacian_eigenmap_1}) for Laplacian eigenmap, the optimization for obtaining the one-dimensional embedding is:
\begin{align}\label{equation_optimization_LPP_1}
\underset{\b{Y}}{\text{minimize}} \quad \sum_{i=1}^n \sum_{j=1}^n w_{ij}\, (y_i - y_j)^2,
\end{align}
for the same interpretation that we had in Laplacian eigenmap. 

\begin{proposition}[\cite{he2004locality}]\label{proposition_LPP_y_L_y}
We have:
\begin{align}
\frac{1}{2} \sum_{i=1}^n \sum_{j=1}^n w_{ij} (y_i - y_j)^2 = \b{u}^\top \b{X} \b{L} \b{X}^\top \b{u}.
\end{align}
\end{proposition}
\begin{proof}
The proof is very similar to the proof of Proposition \ref{proposition_spectral_clustering_f_L_f}. We have:
\begin{align*}
&\b{u}^\top \b{X} \b{L} \b{X}^\top \b{u} \overset{(\ref{equation_Laplacian_matrix})}{=} \b{u}^\top \b{X} (\b{D} - \b{W}) \b{X}^\top \b{u} \\
&= \b{u}^\top \b{X} \b{D} \b{X}^\top \b{u} - \b{u}^\top \b{X} \b{W} \b{X}^\top \b{u} \\
&\overset{(a)}{=} \sum_{i=1}^n \b{u}^\top \b{x}_i\, d_i\, \b{x}_i^\top \b{u} - \sum_{i=1}^n \sum_{j=1}^n \b{u}^\top \b{x}_i\, w_{ij}\, \b{x}_j^\top \b{u} \\
&= \frac{1}{2} \Big(2\sum_{i=1}^n \b{u}^\top \b{x}_i\, d_i\, \b{x}_i^\top \b{u} - 2\sum_{i=1}^n \sum_{j=1}^n \b{u}^\top \b{x}_i\, w_{ij}\, \b{x}_j^\top \b{u}\Big) \\
&\overset{(b)}{=} \frac{1}{2} \Big(2\sum_{i=1}^n (\b{u}^\top \b{x}_i)^2 d_i - 2\sum_{i=1}^n \sum_{j=1}^n (\b{u}^\top \b{x}_i) (\b{u}^\top \b{x}_j) w_{ij}\Big) \\
&\overset{(c)}{=} \frac{1}{2} \Big(\sum_{i=1}^n (\b{u}^\top \b{x}_i)^2 d_i + \sum_{j=1}^n (\b{u}^\top \b{x}_j)^2 d_j \\
&~~~~~~~~~~~~~~~~~~~~~~~~~~~ - 2\sum_{i=1}^n \sum_{j=1}^n (\b{u}^\top \b{x}_i) (\b{u}^\top \b{x}_j) w_{ij}\Big) 
\end{align*}
\begin{align*}
&\overset{(\ref{equation_degree_matrix})}{=} \frac{1}{2} \Big(\sum_{i=1}^n (\b{u}^\top \b{x}_i)^2 \sum_{j=1}^n w_{ij} + \sum_{j=1}^n (\b{u}^\top \b{x}_j)^2 \sum_{i=1}^n w_{ji} \\
&~~~~~~~~~~~~~~~~~~~~~~~~~~~ - 2\sum_{i=1}^n \sum_{j=1}^n (\b{u}^\top \b{x}_i) (\b{u}^\top \b{x}_j) w_{ij} \Big) \\
&\overset{(d)}{=} \frac{1}{2} \Big(\sum_{i=1}^n (\b{u}^\top \b{x}_i)^2 \sum_{j=1}^n w_{ij} + \sum_{j=1}^n (\b{u}^\top \b{x}_j)^2 \sum_{i=1}^n w_{ij} \\
&~~~~~~~~~~~~~~~~~~~~~~~~~~~ - 2\sum_{i=1}^n \sum_{j=1}^n (\b{u}^\top \b{x}_i) (\b{u}^\top \b{x}_j) w_{ij} \Big) \\
&= \frac{1}{2} \sum_{i=1}^n \sum_{j=1}^n w_{ij} (\b{u}^\top \b{x}_i - \b{u}^\top \b{x}_j)^2 \\
&\overset{(\ref{equation_LPP_projection_to_line_2})}{=} \frac{1}{2} \sum_{i=1}^n \sum_{j=1}^n w_{ij} (y_i - y_j)^2,
\end{align*}
where $(a)$ notices that $\b{D}$ is a diagonal matrix, $(b)$ is because $\b{u}^\top \b{x}_i = \b{x}_i^\top \b{u}_i \in \mathbb{R}$, $(c)$ is because we can replace the dummy index $i$ with $j$, and $(d)$ notices that $w_{ij} = w_{ji}$ as we usually use symmetric similarity measures. Q.E.D.
\end{proof}

According to Proposition \ref{proposition_LPP_y_L_y}, we restate Eq. (\ref{equation_optimization_LPP_1}) with adding a constraint inspired by Eq. (\ref{equation_optimization_Laplacian_eigenmap_3}) as:
\begin{equation}\label{equation_optimization_LPP_2}
\begin{aligned}
& \underset{\b{u}}{\text{minimize}}
& & \b{u}^\top \b{X} \b{L} \b{X}^\top \b{u} \\
& \text{subject to}
& & 
\b{u}^\top \b{X} \b{D} \b{X}^\top \b{u} = 1,
\end{aligned}
\end{equation}
where $\b{L}$ and $\b{D}$ are the Laplacian and degree matrix of the adjacency matrix, respectively. Note that the objective variable has been changed to the projection vector $\b{u}$.

The Lagrangian for this optimization problem is \cite{boyd2004convex}:
\begin{align*}
\mathcal{L} = \b{u}^\top \b{X} \b{L} \b{X}^\top \b{u} - \lambda^\top (\b{u}^\top \b{X} \b{D} \b{X}^\top \b{u} - 1),
\end{align*}
where $\lambda$ is the dual variable. Setting the derivative of Lagrangian to zero gives:
\begin{align}
&\mathbb{R}^{d} \ni \frac{\partial \mathcal{L}}{\partial \b{u}} = 2 \b{X} \b{L} \b{X}^\top \b{u} - 2 \lambda \b{X} \b{D} \b{X}^\top \b{u} \overset{\text{set}}{=} \b{0} \nonumber \\
&~~~~~~~~~~~ \implies \b{X} \b{L} \b{X}^\top \b{u} = \lambda \b{X} \b{D} \b{X}^\top \b{u},
\end{align}
which is the generalized eigenvalue problem $(\b{X} \b{L} \b{X}^\top, \b{X} \b{D} \b{X}^\top)$ where $\b{u}$ is the eigenvector \cite{ghojogh2019eigenvalue}. The embedded points are obtained using Eq. (\ref{equation_LPP_projection_to_line}).

\subsubsection{Multi-dimensional Subspace}

Now, we extend LPP to linear projection onto a span of $p$ basis vectors. In other words, we extend it to have $p$-dimensional subspace.
The embedding of $\b{x}_i$ is denoted by $\b{y}_i \in \mathbb{R}^p$. In LPP, we put the embedding vectors column-wise in matrix $\b{Y} = [\b{y}_1, \b{y}_2, \dots, \b{y}_n] \in \mathbb{R}^{p \times n}$.
In this case, the projection is:
\begin{align}\label{equation_LPP_projection_to_subspace}
\mathbb{R}^{p \times n} \ni \b{Y} := \b{U}^\top \b{X},
\end{align}
where $\b{U} = [\b{u}_1, \dots, \b{u}_p] \in \mathbb{R}^{d \times p}$ is the projection matrix onto its column space. 

The Eq. (\ref{equation_optimization_Laplacian_eigenmap_2}) can be extended to multi-dimensional subspace as:
\begin{equation}\label{equation_optimization_Laplacian_eigenmap_2_multi}
\begin{aligned}
& \underset{\b{U}}{\text{minimize}}
& & \textbf{tr}(\b{U}^\top \b{X} \b{L} \b{X}^\top \b{U}) \\
& \text{subject to}
& & 
\b{U}^\top \b{X} \b{D} \b{X}^\top \b{U} = \b{I}.
\end{aligned}
\end{equation}
In optimization, it is assumed that the dimensionality of $\b{U}$ is $d \times d$ but we will truncate it to $d \times p$ after solving the optimization. 

The Lagrangian for this optimization problem is \cite{boyd2004convex}:
\begin{align*}
\mathcal{L} = \textbf{tr}(\b{U}^\top \b{X} \b{L} \b{X}^\top \b{U}) - \textbf{tr}(\b{\Lambda}^\top (\b{U}^\top \b{X} \b{D} \b{X}^\top \b{U} - \b{I})),
\end{align*}
where $\b{\Lambda}$ is the diagonal matrix with the dual variables. Setting the derivative of Lagrangian to zero gives:
\begin{align}
&\mathbb{R}^{d \times d} \ni \frac{\partial \mathcal{L}}{\partial \b{D}} = 2 \b{X} \b{L} \b{X}^\top \b{U} - 2 \b{X} \b{D} \b{X}^\top \b{U} \b{\Lambda} \overset{\text{set}}{=} \b{0} \nonumber \\
&~~~~~~~~~~~ \implies \b{X} \b{L} \b{X}^\top \b{U} = \b{X} \b{D} \b{X}^\top \b{U} \b{\Lambda}, \label{equation_LPP_multi_solution}
\end{align}
which is the generalized eigenvalue problem $(\b{X} \b{L} \b{X}^\top, \b{X} \b{D} \b{X}^\top)$ where columns of $\b{U}$ are the eigenvectors \cite{ghojogh2019eigenvalue}. 
As it is a minimization problem, the columns of $\b{U}$, which are the eigenvectors, are sorted from the corresponding smallest to largest eigenvalues. Moreover, as we have Laplacian matrix, we have one eigenvector with eigenvalue zero which should be ignored (see Section \ref{section_eig_Laplacian}). 
The embedded points are obtained using Eq. (\ref{equation_LPP_projection_to_subspace}).

\subsection{Kernel Locality Preserving Projection}

Kernel LPP, also called kernel supervised LPP,  \cite{cheng2005supervised,li2008kernel} performs LPP in the feature space. 
Assume $\b{\phi}(.)$ denotes the pulling function to the feature space. 
According to representation theory \cite{alperin1993local}, any pulled solution (direction) $\b{\phi}(\b{u}) \in \mathcal{H}$ must lie in the span of all the training vectors pulled to $\mathcal{H}$, i.e., $\b{\Phi}(\b{X}) = [\b{\phi}(\b{x}_1), \dots, \b{\phi}(\b{x}_n)] \in \mathbb{R}^{t\times n}$. 
Hence, $\mathbb{R}^{t} \ni \b{\phi}(\b{u}) = \sum_{i=1}^n \theta_i\, \b{\phi}(\b{x}_i) = \b{\Phi}(\b{X})\, \b{\theta}$ where $\mathbb{R}^n \ni \b{\theta} = [\theta_1, \dots, \theta_n]^\top$ is the unknown vector of coefficients.
Considering multiple projection directions, we have:
\begin{align}\label{equation_Phi_U}
\mathbb{R}^{t \times p} \ni \b{\Phi}(\b{U}) = \b{\Phi}(\b{X})\, \b{\Theta},
\end{align}
where $\mathbb{R}^{n \times p} \ni \b{\Theta} = [\b{\theta}_1, \dots, \b{\theta}_p]$.
Therefore, the objective function in Eq. (\ref{equation_optimization_Laplacian_eigenmap_2_multi}) is converted to:
\begin{align*}
& \textbf{tr}(\b{\Phi}(\b{U})^\top \b{\Phi}(\b{X}) \b{L} \b{\Phi}(\b{X})^\top \b{\Phi}(\b{U})) \\
& \overset{(\ref{equation_Phi_U})}{=} \textbf{tr}(\b{\Theta}^\top \b{\Phi}(\b{X})^\top \b{\Phi}(\b{X}) \b{L} \b{\Phi}(\b{X})^\top \b{\Phi}(\b{X})\, \b{\Theta}) \\
&= \textbf{tr}(\b{\Theta}^\top \b{K}_x \b{L} \b{K}_x\, \b{\Theta}),
\end{align*}
where:
\begin{align}\label{equation_kernel_of_data}
\mathbb{R}^{n \times n} \ni \b{K}_x := \b{\Phi}(\b{X})^\top \b{\Phi}(\b{X}),
\end{align}
is the kernel of data \cite{hofmann2008kernel}. 
Similarly, the constraint of Eq. (\ref{equation_optimization_Laplacian_eigenmap_2_multi}) is converted to $\b{\Theta}^\top \b{K}_x \b{D} \b{K}_x\, \b{\Theta}$.
Finally, the Eq. (\ref{equation_optimization_Laplacian_eigenmap_2_multi}) is changed to:
\begin{equation}\label{equation_optimization_LPP_kernel}
\begin{aligned}
& \underset{\b{\b{\Theta}}}{\text{minimize}}
& & \textbf{tr}(\b{\Theta}^\top \b{K}_x \b{L} \b{K}_x\, \b{\Theta}) \\
& \text{subject to}
& & 
\b{\Theta}^\top \b{K}_x \b{D} \b{K}_x\, \b{\Theta} = \b{I},
\end{aligned}
\end{equation}
where $\b{\Theta}$ is the optimization variable. 
The solution to this problem is similar to the solution of Eq. (\ref{equation_optimization_Laplacian_eigenmap_2_multi}). So, its solution is the generalized eigenvalue problem $(\b{K}_x \b{L} \b{K}_x^\top, \b{K}_x \b{D} \b{K}_x^\top)$ where columns of $\b{\Theta}$ are the eigenvectors \cite{ghojogh2019eigenvalue}. 
Sorting the matrix $\b{\Theta}$ and ignoring the eigenvector with zero eigenvalue (see Section \ref{section_eig_Laplacian}) is done as explained before. The $p$ smallest eigenvectors form $\b{\Theta} \in \mathbb{R}^{n \times p}$. 
The embedded data are obtained as:
\begin{align}
\mathbb{R}^{p \times n} \ni \b{Y} &\overset{(\ref{equation_LPP_projection_to_subspace})}{=} \b{\Phi}(\b{U})^\top \b{\Phi}(\b{X}) \overset{(\ref{equation_Phi_U})}{=} \b{\Theta}^\top \b{\Phi}(\b{X})^\top \b{\Phi}(\b{X}) \nonumber \\
&\overset{(\ref{equation_kernel_of_data})}{=} \b{\Theta}^\top \b{K}_x. \label{equation_kernel_LPP_embedding}
\end{align}

\subsection{Other Improvements over Locality Preserving Projection}

There have been many improvements over the basic LPP. Some examples of these developments are supervised LPP \cite{wong2012supervised}, extended LPP \cite{shikkenawis2012improving}, multiview uncorrelated LPP \cite{yin2019multiview}, graph-optimized LPP \cite{zhang2010graph}, and LPP for Grassmann manifold \cite{wang2017locality}. 
For image recognition purposes, a two-dimensional (2D) LPP, \cite{chen20072d} is proposed whose robust version is \cite{chen20192drlpp} and some of its applications are \cite{hu2007two,xu2009one}.
A survey on LPP variants can be found in \cite{shikkenawis2016some}. 
LPP has been shown to be effective for different applications such as document clustering \cite{cai2005document}, face recognition \cite{he2003learning,yu2006face,yang2017discriminant}, and speech recognition \cite{tang2008study}.

\section{Graph Embedding}\label{section_graph_embedding}

Graph Embedding (GE) \cite{yan2005graph,yan2006graph} is a generalized dimensionality reduction method where the graph of data is embedded in the lower dimensional space \cite{chang2014graph}. 
Note that graph embedding is referred to as two different families of methods in the literature. In the first sense, including GE \cite{yan2005graph,yan2006graph}, it refers to embedding the graph of data in the lower dimensional space. The second sense, however, refers to embedding every node of data graph in a subspace \cite{goyal2018graph}. Here, our focus is on embedding the whole data graph. 
Note that there exist some other generalized subspace learning methods, such as Roweis discriminant analysis \cite{ghojogh2020generalized}, in the literature.

\subsection{Direct Graph Embedding}

Consider the high dimensional data points $\{\b{x}_i \in \mathbb{R}^d\}_{i=1}^n$. We want to reduce the dimensionality of data from $d$ to $p \leq d$ to have $\{\b{y}_i \in \mathbb{R}^p\}_{i=1}^n$, where usually $p \ll d$.
We denote $\b{X} := [\b{x}_1, \dots, \b{x}_n] \in \mathbb{R}^{d \times n}$ and $\b{Y} := [\b{y}_1, \dots, \b{y}_n] \in \mathbb{R}^{p \times n}$.
Similar to spectral clustering, we create an adjacency matrix using Eq. (\ref{equation_adjacency_matrix}). The elements $w_{ij}$ can be determined by the RBF kernel, i.e. Eq. (\ref{equation_RBF_kernel}), or the simple-minded approach, i.e. Eq. (\ref{equation_simple_minded}).

The objective function of optimization in direct GE \cite{yan2005graph,yan2006graph} is inspired by Eq. (\ref{equation_minimization_spectral_clustering_w_f}) or Eq. (\ref{equation_spectral_clustering_optimization_multiple_2}) in spectral clustering:
\begin{equation}\label{equation_optimization_direct_graph_embedding}
\begin{aligned}
& \underset{\b{Y}}{\text{minimize}}
& & \sum_{i=1}^n \sum_{j=1}^n w_{ij}\, \|\b{y}_i - \b{y}_j\|_2^2 \\
& \text{subject to}
& & 
\b{Y}^\top \b{B} \b{Y} = \b{I},
\end{aligned}
\end{equation}
where $\b{B} \succeq 0$ is called the constraint matrix. 
According to Proposition \ref{proposition_spectral_clustering_f_L_f}, this optimization problem is equivalent to:
\begin{equation}\label{equation_optimization_direct_graph_embedding_2}
\begin{aligned}
& \underset{\b{Y}}{\text{minimize}}
& & \b{Y}^\top \b{L} \b{Y} \\
& \text{subject to}
& & 
\b{Y}^\top \b{B} \b{Y} = \b{I},
\end{aligned}
\end{equation}
where $\b{L}$ is the Laplacian of graph of data. 

The Lagrangian for this optimization problem is \cite{boyd2004convex}:
\begin{align*}
\mathcal{L} = \b{Y}^\top \b{L} \b{Y} - \textbf{tr}(\b{\Lambda}^\top (\b{Y}^\top \b{B} \b{Y} - \b{I})),
\end{align*}
where $\b{\Lambda}$ is the diagonal matrix with the dual variables. Setting the derivative of Lagrangian to zero gives:
\begin{align}
&\mathbb{R}^{n \times n} \ni \frac{\partial \mathcal{L}}{\partial \b{Y}} = 2 \b{L} \b{Y} - 2 \b{B} \b{Y} \b{\Lambda} \overset{\text{set}}{=} \b{0} \nonumber \\
&~~~~~~~~~~~ \implies \b{L} \b{Y} = \b{B} \b{Y} \b{\Lambda},
\end{align}
which is the generalized eigenvalue problem $(\b{L}, \b{B})$ where columns of $\b{Y}$ are the eigenvectors \cite{ghojogh2019eigenvalue}. 
Sorting the matrix $\b{Y}$ and ignoring the eigenvector with zero eigenvalue (see Section \ref{section_eig_Laplacian}) is done as explained before. 

\subsection{Linearized Graph Embedding}

As we had for LPP, linearized GE \cite{yan2005graph,yan2006graph} assumes that the embedded data can be obtained by a linear projection onto the column space of a projection matrix $\b{U} \in \mathbb{R}^{d \times p}$. Hence, linearized GE uses similar optimization to Eq. (\ref{equation_optimization_Laplacian_eigenmap_2_multi}):
\begin{equation}\label{equation_optimization_linearized_graph_embedding}
\begin{aligned}
& \underset{\b{U}}{\text{minimize}}
& & \textbf{tr}(\b{U}^\top \b{X} \b{L} \b{X}^\top \b{U}) \\
& \text{subject to}
& & 
\b{U}^\top \b{X} \b{B} \b{X}^\top \b{U} = \b{I},
\end{aligned}
\end{equation}
whose solution is the generalized eigenvalue problem $(\b{X} \b{L} \b{X}^\top, \b{X} \b{B} \b{X}^\top)$, similar to Eq. (\ref{equation_LPP_multi_solution}).

Another approach for linearized GE is a slight revision to Eq. (\ref{equation_optimization_linearized_graph_embedding}) which is \cite{yan2005graph}:
\begin{equation}\label{equation_optimization_linearized_graph_embedding2}
\begin{aligned}
& \underset{\b{U}}{\text{minimize}}
& & \textbf{tr}(\b{U}^\top \b{X} \b{L} \b{X}^\top \b{U}) \\
& \text{subject to}
& & 
\b{U}^\top \b{U} = \b{I},
\end{aligned}
\end{equation}
whose solution is the eigenvalue decomposition of $\b{X} \b{L} \b{X}^\top$. 

\subsection{Kernelized Graph Embedding}

Kernelized GE \cite{yan2005graph,yan2006graph} performs the linearized GE in the feature space. It uses a similar optimization problem as Eq. (\ref{equation_optimization_LPP_kernel}):
\begin{equation}\label{equation_optimization_kernelized_GE}
\begin{aligned}
& \underset{\b{\b{\Theta}}}{\text{minimize}}
& & \textbf{tr}(\b{\Theta}^\top \b{K}_x \b{L} \b{K}_x\, \b{\Theta}) \\
& \text{subject to}
& & 
\b{\Theta}^\top \b{K}_x \b{B} \b{K}_x\, \b{\Theta} = \b{I},
\end{aligned}
\end{equation}
whose solution is the generalized eigenvalue problem $(\b{K}_x \b{L} \b{K}_x^\top, \b{K}_x \b{B} \b{K}_x^\top)$ where columns of $\b{\Theta}$ are the eigenvectors \cite{ghojogh2019eigenvalue}.
Sorting the matrix $\b{\Theta}$ and ignoring the eigenvector with zero eigenvalue (see Section \ref{section_eig_Laplacian}) is done as explained before. The embedding is then obtained using Eq. (\ref{equation_kernel_LPP_embedding}). 

Another approach for linearized GE is a slight revision to Eq. (\ref{equation_optimization_kernelized_GE}). 
This approach is the kernelized version of Eq. (\ref{equation_optimization_linearized_graph_embedding2}) and is \cite{yan2005graph}:
\begin{equation}\label{equation_optimization_kernelized_GE_2}
\begin{aligned}
& \underset{\b{\b{\Theta}}}{\text{minimize}}
& & \textbf{tr}(\b{\Theta}^\top \b{K}_x \b{L} \b{K}_x\, \b{\Theta}) \\
& \text{subject to}
& & 
\b{\Theta}^\top \b{K}_x \b{\Theta} = \b{I},
\end{aligned}
\end{equation}
whose solution is the generalized eigenvalue problem $(\b{K}_x \b{L} \b{K}_x^\top, \b{K}_x)$.

\subsection{Special Cases of Graph Embedding}

GE is a generalized subspace learning method whose special cases are some of the well-known dimensionality reduction methods. These special cases are cases of either direct, or linearized, or kernelized GE. Some of the special cases are Laplacian eigenmap, LPP, kernel LPP, Principal Component Analysis (PCA), kernel PCA, Fisher Discriminant Analysis (FDA), kernel FDA, Multidimensional Scaling (MDS), Isomap, and Locally Linear Embedding (LLE). 

\subsubsection{Laplacian Eigenmap}

Comparing Eq. (\ref{equation_optimization_Laplacian_eigenmap_2}) of Laplacian eigenmap with Eq. (\ref{equation_optimization_direct_graph_embedding_2}), we see that $\b{B} = \b{I}$. Also comparing Eq. (\ref{equation_optimization_Laplacian_eigenmap_3}) as the second approach of Laplacian eigenmap, with Eq. (\ref{equation_optimization_direct_graph_embedding_2}) shows that $\b{B} = \b{D}$. Hence, Laplacian eigenmap is a special case of direct GE.

\subsubsection{LPP and Kernel LPP}

Comparing Eq. (\ref{equation_optimization_Laplacian_eigenmap_2_multi}) with Eq. (\ref{equation_optimization_linearized_graph_embedding}) and also comparing Eq. (\ref{equation_optimization_LPP_kernel}) with Eq. (\ref{equation_optimization_kernelized_GE}) show that LPP and kernel LPP are special cases of linearized GE and kernelized GE, respectively. 

\subsubsection{PCA and Kernel PCA}

The optimization of PCA is \cite{ghojogh2019unsupervised}:
\begin{equation}\label{equation_PCA}
\begin{aligned}
& \underset{\b{U}}{\text{minimize}}
& & \textbf{tr}(\b{U}^\top \b{S} \b{U}) \overset{(a)}{=} \textbf{tr}(\b{U}^\top \b{X} \b{X}^\top \b{U}) \\
& \text{subject to}
& & 
\b{U}^\top \b{U} = \b{I},
\end{aligned}
\end{equation}
where $\b{U}$ is the projection matrix, $\b{S}$ is the covariance matrix of data, and $(a)$ is because of $\b{S} = \b{X} \b{X}^\top$ assuming that data $\b{X}$ are already centered. 
Moreover, the solution of kernel PCA is the eigenvalue decomposition of kernel of data \cite{ghojogh2019unsupervised}; hence, optimization of kernel PCA is the kernelized version of Eq. (\ref{equation_PCA}). 
Comparing Eq. (\ref{equation_PCA}) with Eq. (\ref{equation_optimization_linearized_graph_embedding2}) shows that PCA is a special case of linearized GE with $\b{L}=\b{I}$. Moreover, kernel PCA is a special case of kernelized GE with $\b{L} = \b{I}$. 

\subsubsection{FDA and Kernel FDA}

The optimization of FDA is \cite{ghojogh2019fisher,fisher1936use}:
\begin{equation}\label{equation_optimization_FDA_with_S_B}
\begin{aligned}
& \underset{\b{U}}{\text{maximize}}
& & \textbf{tr}(\b{U}^\top \b{S}_B\, \b{U}) \\
& \text{subject to}
& & \b{U}^\top \b{S}_W\, \b{U} = \b{I},
\end{aligned}
\end{equation}
where $\b{S}_B$ and $\b{S}_W$ are the between- and within-class scatters, respectively. 
We also have \cite{ye2007least}: 
\begin{align}\label{equation_S_T_as_sum_of_scatters}
\b{S} = \b{S}_B + \b{S}_W \implies \b{S}_B = \b{S} - \b{S}_W,
\end{align}
where $\b{S}$ denotes the total covariance of data. Hence, the Fisher criterion, to be maximized, can be restated as:
\begin{align}\label{equation_Fisher_criterion}
& \frac{\textbf{tr}(\b{U}^\top \b{S}_B\, \b{U})}{\textbf{tr}(\b{U}^\top \b{S}_W\, \b{U})} = \frac{\textbf{tr}(\b{U}^\top \b{S}\, \b{U})}{\textbf{tr}(\b{U}^\top \b{S}_W\, \b{U})} - 1,
\end{align}
whose second constant term can be ignored in maximization. Therefore, Eq. (\ref{equation_optimization_FDA_with_S_B}) is restated to \cite{yan2005graph,ghojogh2020generalized}:
\begin{equation}\label{equation_optimization_FDA_with_S_T}
\begin{aligned}
& \underset{\b{U}}{\text{maximize}}
& & \textbf{tr}(\b{U}^\top \b{S}\, \b{U}) \overset{(a)}{=} \textbf{tr}(\b{U}^\top \b{X} \b{X}^\top \b{U}) \\
& \text{subject to}
& & \b{U}^\top \b{S}_W\, \b{U} = \b{I},
\end{aligned}
\end{equation}
where $(a)$ is because of $\b{S} = \b{X} \b{X}^\top$ assuming that data $\b{X}$ are already centered. 

Let $c_i$ denote the class to which the point $\b{x}_i$ belongs and $\b{\mu}_{c_i}$ be the mean of class $c_i$. Let $c$ denote the number of classes and $n_j$ be the sample size of class $j$. For every class $j$, we define $\b{e}_j = [e_{j1}, \dots, e_{jn}]^\top \in \mathbb{R}^n$ as:
\begin{align}
e_{ji} = 
\left\{
    \begin{array}{ll}
        1 & \mbox{if } c_i = j, \\
        0 & \mbox{otherwise}.
    \end{array}
\right.
\end{align}
Note that the within scatter can be obtained as \cite{yan2005graph}:
\begin{align}\label{equation_within_scatter_restate_for_GE}
\b{S}_W &= \sum_{i=1}^n (\b{x}_i - \b{\mu}_{c_i}) (\b{x}_i - \b{\mu}_{c_i})^\top \nonumber \\
&= \b{X} (\b{I} - \sum_{j=1}^{c} \frac{1}{n_j} \b{e}_j \b{e}_j^\top) \b{X}^\top.
\end{align}
Comparing Eq. (\ref{equation_optimization_FDA_with_S_T}) with Eq. (\ref{equation_optimization_linearized_graph_embedding}) and noticing Eq. (\ref{equation_within_scatter_restate_for_GE}) show that FDA is a special case of linearized GE with $\b{L} = \b{I}$ and $\b{B} = \b{I} - \sum_{j=1}^{c} (1/n_j) \b{e}_j \b{e}_j^\top$. 
Kernel FDA \cite{ghojogh2019fisher,mika1999fisher} is the kernelized version of Eq. (\ref{equation_optimization_FDA_with_S_T}) so it is a special case of kernelized GE with the mentioned $\b{L}$ and $\b{B}$ matrices. 

\subsubsection{MDS and Isomap}

Kernel classical MDS uses a kernel as \cite{ghojogh2020multidimensional}:
\begin{align}\label{equation_Isomap_kernel}
\mathbb{R}^{n \times n} \ni \b{K} = -\frac{1}{2} \b{H} \b{D} \b{H},
\end{align}
where $\b{D} \in \mathbb{R}^{n \times n}$ is a matrix with squared Euclidean distance between points and $\mathbb{R}^{n \times n} \ni \b{H} := \b{I} - (1/n) \b{1} \b{1}^\top$ is the centering matrix. 
Isomap also applies multidimensional scaling with a geodesic kernel which uses piece-wise Euclidean distance for computing $\b{D}$ \cite{tenenbaum2000global,ghojogh2020multidimensional}.
The row summation of this kernel matrix is \cite{yan2005graph,yan2006graph}: 
\begin{align*}
&\sum_{j=1}^n \b{K}_{ij} \overset{(\ref{equation_Isomap_kernel})}{=} \sum_{j=1}^n (-\frac{1}{2} \b{H} \b{D} \b{H})_{ij} \\
&= \sum_{j=1}^n \big(-\frac{1}{2} (\b{I} - \frac{1}{n} \b{1} \b{1}^\top) \b{D} (\b{I} - \frac{1}{n} \b{1} \b{1}^\top)\big)_{ij} \\
&= \frac{1}{2} \sum_{j=1}^n \big( -\b{D}_{ij} + \frac{1}{n} \sum_{i'=1}^n \b{D}_{i' j} \\
&~~~~~~~~~~~~~~~~~~~~~~~~~~~~~~ + \frac{1}{n} \sum_{j'=1}^n (\b{D}_{i j'} - \frac{1}{n} \sum_{k'=1}^n \b{D}_{k' j'}) \big) \\
&= (-\frac{1}{2} \sum_{j=1}^n \b{D}_{ij} + \frac{1}{2n} \sum_{j=1}^n \sum_{i'=1}^n \b{D}_{i' j}) + (\frac{1}{2n} \sum_{j=1}^n \sum_{j'=1}^n \b{D}_{i j'} \\
&~~~~~~ - \frac{1}{2n^2} \sum_{j=1}^n \sum_{j'=1}^n \sum_{k'=1}^n \b{D}_{k' j'}) = \b{0} + \b{0} = \b{0}.
\end{align*}
According to Eq. (\ref{equation_Laplacian_row_sum}), the kernel used in kernel classical MDS or the geodesic kernel used in Isomap can be interpreted as the Laplacian matrix of graph of data as it satisfies its row-sum property. 
The optimization of MDS or Isomap is \cite{ghojogh2020multidimensional}:
\begin{equation}\label{equation_optimization_Isomap}
\begin{aligned}
& \underset{\b{Y}}{\text{minimize}}
& & \b{Y}^\top \b{K} \b{Y} \\
& \text{subject to}
& & 
\b{Y}^\top \b{Y} = \b{I},
\end{aligned}
\end{equation}
with kernel of Eq. (\ref{equation_Isomap_kernel}). 
Comparing Eq. (\ref{equation_optimization_Isomap}) with Eq. (\ref{equation_optimization_direct_graph_embedding_2}) shows that kernel classical MDS and Isomap are special cases of direct GE with $\b{L} = \b{K}$ and $\b{B} = \b{I}$. 

\subsubsection{LLE}

Let $\b{W} \in \mathbb{R}^{n \times n}$ be the reconstruction weight matrix in LLE \cite{roweis2000nonlinear,ghojogh2020locally}. Let $\mathbb{R}^{n \times n} \ni \b{M} := (\b{I} - \b{W}) (\b{I} - \b{W})^\top$.
The row summation of matrix $\b{M}$ is \cite{yan2005graph,yan2006graph}:
\begin{align*}
\sum_{j=1}^n \b{M}_{ij} &= \sum_{j=1}^n \big((\b{I} - \b{W}) (\b{I} - \b{W})^\top\big)_{ij} \\
&= \sum_{j=1}^n \big(\b{I}_{ij} - \b{W}_{ij} - \b{W}_{ji} + (\b{W}\b{W}^\top)_{ij}\big) \\
&= \b{1} - \sum_{j=1}^n (\b{W}_{ij} + \b{W}_{ji}) + \sum_{j=1}^n \sum_{k=1}^n (\b{W}_{ij} \b{W}_{jk}) \\
&= \b{1} - \sum_{j=1}^n (\b{W}_{ij} + \b{W}_{ji}) + \sum_{j=1}^n (\b{W}_{ij} \sum_{k=1}^n \b{W}_{jk}) \\
&\overset{(a)}{=} \b{1} - \sum_{j=1}^n \b{W}_{ij} - \sum_{j=1}^n \b{W}_{ji} + \sum_{j=1}^n \b{W}_{ij} = \b{0},
\end{align*}
where $(a)$ is because the row summation of reconstruction matrix is one, i.e. $\sum_{k=1}^n \b{W}_{jk} = \b{1}$, according to the constraint in the linear reconstruction step of LLE (see \cite{ghojogh2020locally} for more details). 
According to Eq. (\ref{equation_Laplacian_row_sum}), the matrix $\b{M}$, used in LLE, can be interpreted as the Laplacian matrix of graph of data as it satisfies its row-sum property. 
The optimization of LLE is \cite{ghojogh2020locally}:
\begin{equation}\label{equation_optimization_LLE}
\begin{aligned}
& \underset{\b{Y}}{\text{minimize}}
& & \textbf{tr}(\b{Y}^\top\b{M}\b{Y}), \\
& \text{subject to}
& & \frac{1}{n} \b{Y}^\top \b{Y} = \b{I}, \\
& & & \b{Y}^\top \b{1} = \b{0},
\end{aligned}
\end{equation}
whose second constraint can be ignored because it is automatically satisfied (see \cite{ghojogh2020locally}). 
Comparing Eq. (\ref{equation_optimization_LLE}), without its second constraint, to Eq. (\ref{equation_optimization_direct_graph_embedding_2}) shows that LLE is a special case of direct GE with $\b{L} = \b{M}$ and $\b{B} = (1/n) \b{I}$.

\subsection{Other Improvements over Graph embedding}

Some of the improved variants of graph embedding are kernel eigenmap for graph embedding with side information \cite{brand2003continuous}, fuzzy graph embedding \cite{wan2017local}, graph embedding with extreme learning machine \cite{yang2019graph}, and deep dynamic graph embedding \cite{goyal2018dyngem,goyal2020dyngraph2vec}. Graph embedding has also been used for domain adaptation \cite{hedegaard2020supervised,hedegaard2020supervised2}. 
A Python package for graph embedding toolbox is available \cite{goyal2018gem}.

\section{Diffusion Map}\label{section_diffusion_map}

Diffusion map, proposed in \cite{Lafon2004diffusion,coifman2005geometric,coifman2006diffusion}, is a nonlinear dimensionality reduction method which makes use of Laplacian of data. 
Note that the literature of diffusion map usually refers to Laplacian as the normalized kernel (see Section \ref{section_Laplacian_matrix_definition} where we explain why Laplacian is sometimes called the normalized kernel). 
Diffusion map reveals the nonlinear underlying manifold of data by a diffusion process, where the structure of manifold gets more obvious by passing some time on running the algorithm. Its algorithm includes several concepts explained in the following. 

\subsection{Discrete Time Markov Chain}

Consider the graph of data, denoted by $\mathcal{X} = \{\b{x}_1, \dots, \b{x}_n\}$ where $\b{x}_i \in \mathbb{R}^d$ is the $i$-th data point in the dataset. Consider a $k$NN graph, or an adjacency matrix, of data. Starting from any point in the graph, we can randomly traverse the data graph to move to another point. This process is known as a random walk on the graph of data \cite{de2008introduction}. Let $\mathbb{P}(\b{x}_i, \b{x}_j)$ denote the probability of moving from point $\b{x}_i$ to $\b{x}_j$. Obviously, if the two points are closer to one another in the graph or adjacency matrix, i.e. if they are more similar, this probability is larger. Movements from a point to another point in this graph can be modeled as a discrete time Markov chain\footnote{A discrete time Markov chain is a sequence of random variables where the value of the next random variable depends only on the value of the current random value and not the previous random variables in the sequence.} \cite{ross2014introduction}. Let $\b{M}$ be the transition matrix of a Markov chain on $\mathcal{X}$ where its elements are the probabilities of transition \cite{ghojogh2019hidden}. We can repeat this transition for any number of times. The transition matrix at the $t$-th step is obtained as $\b{M}^t$ \cite{ross2014introduction}. 

Let $\b{W} \in \mathbb{R}^{n \times n}$ be the adjacency matrix. Diffusion map usually uses the RBF kernel, Eq. (\ref{equation_RBF_kernel}), for the adjacency matrix \cite{nadler2006diffusion,nadler2006diffusion2}. According to Eq. (\ref{equation_Laplacian_matrix_2}), the Laplacian is calculated for a value of $\alpha$:
\begin{align}
\mathbb{R}^{n \times n} \ni \b{L}^{(\alpha)} := \b{D}^{-\alpha} \b{W} \b{D}^{-\alpha}.
\end{align}
Let $\b{D}^{(\alpha)} \in \mathbb{R}^{n \times n}$ denote the diagonal degree matrix for $\b{L}^{(\alpha)}$ which is calculated similarly as Eq. (\ref{equation_degree_matrix}) where $\b{W}$ is replaced by $\b{L}^{(\alpha)}$ in that equation.
Applying the so-called \textit{random-walk graph Laplacian normalization} to the Laplacian matrix gives us \cite{chung1997spectral}:
\begin{align}\label{equation_diffusion_map_transitionMatrix}
\mathbb{R}^{n \times n} \ni \b{M} := (\b{D}^{(\alpha)})^{-1} \b{L}^{(\alpha)}, 
\end{align}
which we take as the transition matrix of our Markov chain here. 
Note that Eq. (\ref{equation_diffusion_map_transitionMatrix}) is a stochastic transition matrix with
non-negative rows which sum to one \cite{brand2003continuous}. 
According to the definition of transition matrix, the $(i,j)$-th element of transition matrix at time step $t$, i.e. $\b{M}^t$, is:
\begin{align}
\mathbb{P}(\b{x}_j, t\, |\, \b{x}_i) = \b{M}^t(i,j).
\end{align}

\subsection{The Optimization Problem}\label{section_diffusion_map_optimization}

At the time step $t$, diffusion map 
calculates the embeddings $\b{Y} \in \mathbb{R}^{n \times d}$ by maximizing $\textbf{tr}(\b{Y}^\top \b{M}^t \b{Y})$ where it assumes the embeddings have scaled identity covariance, i.e. $\b{Y}^\top \b{Y} = \b{I}$. 
The term $\textbf{tr}(\b{Y}^\top \b{M}^t \b{Y})$ shows the transition matrix of data, at time $t$, after projection onto the column space of $\b{Y}$ \cite{ghojogh2019unsupervised}.
The optimization problem is: 
\begin{equation}\label{equation_diffusion_map_optimization}
\begin{aligned}
& \underset{\b{Y}}{\text{maximize}}
& & \textbf{tr}(\b{Y}^\top \b{M}^t \b{Y}) \\
& \text{subject to}
& & 
\b{Y}^\top \b{Y} = \b{I},
\end{aligned}
\end{equation}
whose solution is the eigenvalue problem of $\b{M}^t$:
\begin{align}
\b{M}^t \b{Y} = \b{Y} \b{\Lambda},
\end{align}
where columns of $\b{Y}$ are the eigenvectors \cite{ghojogh2019eigenvalue} and diagonal elements of $\b{\Lambda}$ are the eigenvalues of $\b{M}^t$. 
As it is a maximization problem, the eigenvectors are sorted from the corresponding largest to smallest eigenvalues. 
Then, we truncate the solution by taking the $p$ leading eigenvectors to have the $p$-dimensional embedding $\b{Y} \in \mathbb{R}^{n \times p}$ where $p \leq d$. Note that diffusion map usually scales the embedding using eigenvalues. This will be explained in more detail in Section \ref{section_difusion_distance}.

\subsection{Diffusion Distance}\label{section_difusion_distance}

We can define a diffusion distance measuring how dissimilar two points are with regards to their probability of random walk from their neighbors to them. In other words, this distance, at time step $t$, is defined as \cite{nadler2006diffusion,nadler2006diffusion2}:
\begin{align}
d^{(t)}(\b{x}_i, \b{x}_j) := \sqrt{\sum_{\ell=1}^n \frac{\big(\mathbb{P}(\b{x}_\ell, t | \b{x}_i) - \mathbb{P}(\b{x}_\ell, t | \b{x}_j)\big)^2}{\b{\psi}_1(\ell)}},
\end{align}
where $\b{\psi}_1 = [\b{\psi}_1(1), \dots, \b{\psi}_1(n)] \in \mathbb{R}^n$ is the leading eigenvector with the largest eigenvalue, i.e. $\b{Y} = [\b{\psi}_1, \dots, \b{\psi}_p]$.
The denominator is for the sake of normalization.
It is shown in \cite{nadler2006diffusion2} that this distance is equivalent to the original diffusion distance proposed in \cite{coifman2006diffusion}:
\begin{align}
d^{(t)}(\b{x}_i, \b{x}_j) := \sqrt{\sum_{\ell=1}^n \lambda_\ell^{2t} \big(\b{\psi}_\ell(i) - \b{\psi}_\ell(j)\big)^2},
\end{align}
where $\lambda_\ell$ denotes the $\ell$-th largest eigenvalue. 

If, at time $t$, the diffusion distance of two points is small, it implies that the two points are similar and close to each other in the embedding of step $t$. 
In other words, a small diffusion map between two points means those points probably belong to the same cluster in the embedding. 
Note that, as was explained in Section \ref{section_diffusion_map_optimization}, the diffusion map, i.e. the embedding, is computed as $\b{Y} = [\b{y}_1, \dots, \b{y}_n]^\top \in \mathbb{R}^{n \times p}$ where \cite{coifman2006diffusion}:
\begin{align}
\mathbb{R}^p \ni \b{y}_i^{(t)} := [\lambda_1^t \b{\psi}_1(i), \dots, \lambda_p^t \b{\psi}_p(i)]^\top,
\end{align}
for time step $t$.
Note that diffusion map scales the embedding using the corresponding eigenvalues to the power of time step. Hence, the diffusion map gets larger and larger by going forward in time. That is why the name of this method is diffusion map as it diffuses in time where $t$ is the time step and plays as a parameter for the scale, too. 

Moreover, it is shown in \cite{nadler2006diffusion2} that:
\begin{align}
d^{(t)}(\b{x}_i, \b{x}_j) = \big\|\b{y}_i^{(t)} - \b{y}_j^{(t)}\big\|_2.
\end{align}
Therefore, the diffusion distance is equivalent to the distance of embedded points in the diffusion map.

\subsection{Other Improvements over Diffusion maps}

An improvement over diffusion map is the vector diffusion map \cite{singer2012vector}, which is a generalization of diffusion map and several other nonlinear dimensionality reduction methods. 
Diffusion map has also been used in several applications such as data fusion \cite{lafon2006data}, nonlinear independent component analysis \cite{singer2008non}, and changing data \cite{coifman2014diffusion}. 
The Fokker-Planck operator, which describes the time evolution of a probability density function, has also been used in some variants of diffusion maps \cite{nadler2005diffusion}. 

\section{Conclusion}\label{section_conclusion}

In this paper, we provided a tutorial and survey for the nonlinear dimensionality reduction methods which are based on the Laplacian of graph of data. We introduced the definition of Laplacian and then we covered the most important methods in this family, namely spectral clustering, Laplacian eigenmap, locality preserving projection, graph embedding, and diffusion map. 

\section*{Acknowledgement}

Some parts of this tutorial paper (particularly some parts of spectral clustering and Laplacian eigenmap) have been covered by Prof. Ali Ghodsi's tutorial videos, from University of Waterloo, on YouTube. 

% The authors hugely thank the instructors of ... course at the ... University (you can see their YouTube channel) whose lectures partly covered some materials mentioned in this tutorial paper. 

% \appendix

\bibliography{References}
\bibliographystyle{icml2016}

\end{document}